\documentclass{article}

\usepackage{microtype}
\usepackage{subcaption}
\usepackage{soul}
\usepackage{booktabs} 

\usepackage{hyperref}
\usepackage{authblk}
\usepackage{graphicx,xcolor} 
\usepackage[font=small]{caption}

\usepackage{tikz}
\usetikzlibrary{positioning}
\usepackage{babel}
\usepackage{float}

\newcommand{\Tm}{\mathcal{T}_m}

\newcommand{\y}{\mathbf{y}}
\newcommand{\z}{\mathbf{z}}
\newcommand{\A}{\log \mm}
\newcommand{\R}{\mathbb{R}}
\usepackage{amsfonts,amssymb,amsmath,amsthm,mathtools,thmtools, thm-restate, mathrsfs}

\newcommand{\pe}{\mathbb{E}}
\newcommand{\RELU}{\mathsf{RELU()}}
\newcommand{\EQ}{\mathsf{EQ}}
\newcommand{\Tf}{\mathsf{Tf}}

\newcommand{\mm}{{\lceil m/2\rceil}}

\newcommand{\fb}{\mathsf{float32}}

\usepackage[preprint]{icml2026}

\usepackage[capitalize,noabbrev]{cleveref}

\theoremstyle{plain}
\newtheorem{theorem}{Theorem}[section]

\newtheorem{lemma}[theorem]{Lemma}

\theoremstyle{definition}
\newtheorem{definition}[theorem]{Definition}

\theoremstyle{remark}
\newtheorem{remark}[theorem]{Remark}

\newtheorem{example}[remark]{Example}

\usepackage[textsize=tiny]{todonotes}

\icmltitlerunning{Every Bit Counts: A Theoretical Study of Precision–Expressivity Tradeoffs in Quantized Transformers}

\begin{document}

\twocolumn[
  \icmltitle{Every Bit Counts: A Theoretical Study of Precision–Expressivity Tradeoffs in Quantized Transformers}



  \icmlsetsymbol{equal}{*}

  \begin{icmlauthorlist}
    \icmlauthor{Sayak Chakrabarti}{yyy}
    \icmlauthor{Toniann Pitassi}{yyy}
    \icmlauthor{Josh Alman}{yyy}
  \end{icmlauthorlist}

  \icmlaffiliation{yyy}{Department of Computer Science, Columbia University, New York, NY, USA}

  \icmlcorrespondingauthor{Sayak Chakrabarti}{sayaksc@gmail.com}
  \icmlcorrespondingauthor{Toniann Pitassi}{toni@cs.columbia.edu}
  \icmlcorrespondingauthor{Josh Alman}{josh@cs.columbia.edu}

  \icmlkeywords{Machine Learning, ICML}

  \vskip 0.3in
]



\printAffiliationsAndNotice{}  

\begin{abstract}
  Quantization reduces the numerical precision of Transformer computations and is widely used to accelerate inference, yet its effect on expressivity remains poorly characterized. We demonstrate a fine-grained theoretical tradeoff between  expressivity and precision: For every $p$ we exhibit a function $\Gamma$, inspired by the equality function, and prove that a one-layer softmax Transformer can compute $\Gamma$, with $p$ bits of precision, but not with $p\!-\!1$ bits of precision.

  This result concretely explains the widely observed phenomenon of empirical loss of expressivity  when quantization is used. Practically, it suggests that tasks requiring equality-like comparisons (exact match, membership, etc.) are especially sensitive to quantization. Dropping even one bit can cross a threshold where the model cannot represent the needed comparison reliably. Thus, it paves the way for developing heuristics that will help practitioners choose how much quantization is possible: the precision should be chosen as a function of the length of equality to be checked for the specific task.

   Our proofs combine explicit finite-precision Transformer constructions with communication-complexity lower bounds, yielding a tight ``one-bit" threshold.
\end{abstract}

\section{Introduction}

Multi-layer Transformers \cite{vaswani2017attention} are at the core of recent advances in large language models (LLMs). They have become the default backbone for sequence modeling across natural language processing and beyond, including image classification \cite{dosovitskiy2020image}, speech recognition \cite{gulati2020conformer}, protein sequencing \cite{rives2021biological}, and robotics \cite{chen2021decision}. Their success is often attributed to their ability to  capture pair-wise correlations between tokens. However, this comes with a price: the time and space complexity of this model scale quadratically in the sequence length, which requires heavy memory and compute resources for training and inference. 

In order to improve efficiency of these models, there is a large body of work that studies approximation algorithms to reduce the cost of attention and long-context processing (e.g., Longformer \cite{beltagy2020longformer}, BigBird \cite{zaheer2020big}, Reformer \cite{kitaev2020reformer}, Performer \cite{choromanski2020rethinking}), Hyperattention \cite{han2023hyperattention}, KDEformer \cite{zandieh2023kdeformer}.

\paragraph{Quantization.} A popular line of work in this area targets the memory footprint of Transformers by focusing on  systems-level efficiency via compression (low precision approximation). { Prior work has explored parameter sharing and factorized embeddings to reduce redundant storage \cite{lan2019albert}, structured pruning of attention heads to enable efficient mapping to hardware \cite{michel2019sixteen,voita2019analyzing}, systems-aware attention implementations and scheduling to reduce activation and IO overhead \cite{dao2022flashattention,dao2023flashattention}, KV-cache reductions such as multi-query attention and grouped-query attention that substantially shrink per-token state \cite{shazeer2019fast,ainslie2023gqa}. Among compression techniques, \emph{quantization}, the process of limiting the numerical precision of a deep learning model, has emerged especially powerful and practical.} 
It  reduces the arithmetic and memory costs of Transformer inference via low-bit kernels {(hardware-optimized implementations of core operations such that matrix multiplications that operate directly on low precision numbers), and quantized formats that define how tensors are stored using low-bit integers, plus scaling metadata.} Intuitively, quantization should reduce the expressive power of the model, and the goal of this paper is to formally study the tradeoff between quantization and expressivity.

\paragraph{Prior Empirical Work on Quantization.} A large body of work studies reduced-precision deployment, including weight-only or parameter quantization \cite{frantar2022gptq,lin2024awq,dettmers2023spqr}, activation quantization \cite{xiao2023smoothquant,yao2022zeroquant,choi2018pact}, or dynamic or on-the-fly quantization schemes \cite{kim2021bert}, {where the values are quantized but arithmetic is often performed in regular precision.}

{In this paper, we are interested in quantization that also uses low-precision arithmetic, i.e., we treat limited numerical precision on core operations, and ask what functions remain computable when each primitive operation is carried out with restricted bit complexity.} Low-precision arithmetic makes computations much faster, and has been widely used in deep-learning and Transformer literature.

There are two commonly used formats to represent numbers: the \emph{fixed-point precision} format and the \emph{floating-point precision} format. {We will formally define these in \cref{apx:prelims}, but we present an overview here.} The fixed-point precision format for $p$ bits works as: first, the values of the numbers are scaled to bring them within the range $-(2^{p-1}-1)$ to $(2^{p-1}-1)$, and then rounded off to the closest integer $(p-1)$-bit integer. This $(p-1)$-bit integer along with one sign bit gives the representation of the number in $p$-bit fixed-point precision format. Another commonly used precision format in computing is the \emph{floating-point precision} format. A number in this format is represented by a mantissa $\mu$ and an exponent $\varepsilon$, and represents the number $\pm1.\mu\times 2^{\pm\varepsilon}$. Each of the mantissa and the exponent has one sign bit each, giving the total number of bits of precision as $1+|\mu|+1+|\varepsilon|$. Note that floating-point precision format can simultaneously represent much larger and much smaller non-zero values than fixed-point precision as it stores numbers with a variable exponent. However, this often introduces rounding quirks (many decimals like 0.1 aren’t exact), and equality comparisons can be tricky. On the other hand, with fixed-point precision format, every represented value is equally scaled by a constant factor, so the spacing between representable numbers is uniform across the entire range, thereby giving better predictability, but often leading to either large numbers being difficult to represent or small numbers being rounded to zero.

Due to the advantages and disadvantages of both  precision formats, they have been widely studied. Quantized computations with fixed point precision often employ 8 bits or 4 bits, which are represented as INT8 or INT4, respectively. For example, \cite{dettmers2022gpt3,rock2022int8} perform arithmetic for Transformers using INT8  by giving specific algorithms for operations like matrix multiplication and softmax computation. Other examples of fixed-point precision arithmetic in INT8 include \cite{xiao2023smoothquant,yao2022zeroquant,kim2021bert,zafrir2019q8bert} and in INT4 include \cite{frantar2022gptq,lin2024awq,dettmers2023spqr,shao2023omniquant,dettmers2023qlora,frantar2025marlin}.

Quantization to low-bit floating-point format is also commonly performed on Transformers and other neural networks. Usually, quantized floating-point numbers use 16 bit of precision (FP16), or 8 bits of precision (FP8).  Some of the works of quantization on FP16 include \cite{micikevicius2017mixed,shoeybi2019megatron,rajbhandari2020zero,aminabadi2022deepspeed} and on FP8 include \cite{micikevicius2022fp8,peng2023fp8,li2023fp8}.

Some more works study even lower precision, even as low as one-bit quantizations \cite{courbariaux2015binaryconnect,courbariaux2016binarized,hubara2018quantized,zhou2016dorefa,wang2025bitnet,ma2024era}.

Other works that deploy low-precision arithmetic to speed up deep-learning models include \cite{courbariaux2015binaryconnect,rastegari2016xnor,jacob2018quantization,park2022lut,yao2022zeroquant,junczys2018marian,zafrir2019q8bert,bhandare2019efficient,chen2020statistical}.

With the growing popularity of quantization, modern hardware now enables dedicated low-precision matrix multiplication primitives, e.g., NVIDIA Tensor Cores support FP8 as well as INT8/INT4 MMA variants, while CPUs increasingly provide INT8 dot-product and matrix-multiply extensions (e.g., Intel AMX and Arm’s I8MM).

To facilitate quantization on modern hardware, Graphcore, AMD, and Qualcomm published their FP8 specification in \cite{noune20228} and NVIDIA, Arm, and Intel followed \cite{micikevicius2022fp8}. Both works considered FP8 as E4M3 (4 exponent bits, 3 mantissa bits) and E5M2 (5 exponent bits, 2 mantissa bits).

Despite being a standard engineering tool, the theoretical implications of quantization are less clear: what capabilities are preserved by lowering precision, and what capabilities are fundamentally lost? As observed in previous works, quantization comes with a price in terms of a loss in expressivity. \cite{yao2022zeroquant} reports that for BERT-base on GLUE\footnote{GLUE \cite{wang2018glue} is a collection of NLP tasks that requires ``understanding'' by the language model.}, quantizing from full precision to INT8 activations drops the average accuracy from $84.06\%$ to $79.61\%$, and quantizing to even lower precision like INT4/INT8 and activations to FP16 drops the accuracy to as low as $33.11\%$. Similarly, outliers (activation values that occasionally spike to very large numbers) and their adverse impact on quantization was observed by \cite{wei2023outlier}, where they quantized BERT-large to INT8 for GLUE benchmark tasks and recorded an accuracy drop of around $13\%$. Our main result, \cref{thm:tradeoff} below, can help explain this phenomenon: we prove that equality checks require high-precision Transformers, which means that achieving high accuracy on datasets like GLUE, which consist of tasks similar to equality testing like checking semantic equivalence and paraphrasing, also needs high precision.

Usually parameters are quantized after training in usual precision format, which is called post-training quantization (PTQ) \cite{xiao2023smoothquant,nagel2020up}. In practice, if PTQ suffers high accuracy loss, the accuracy can sometimes be recovered by fine tuning, e.g., quantization aware training (QAT) where, at every epoch, the model is updated in usual precision and then quantized  \cite{jacob2018quantization,zhou2016dorefa,choi2018pact}, and other fine tuning methods like \cite{dettmers2023qlora,xu2023qa,jeon2025l4q}. 

The success of fine-tuning approaches for improving the accuracy raises a fundamental question: Does quantization necessarily reduce the expressive power of the model, or can such techniques always recover the full expressive power?
This paper addresses that gap. More specifically, we  show that the loss in expressivity is provably unavoidable: we show that even losing a single bit of precision must dramatically reduce a model's expressive power for important tasks, regardless of what training, fine-tuning, or other techniques are used to find the weights. 

\paragraph{Theoretical Understanding of Precision.} Prior work has studied the theory of numerical precision in Transformers, but only at a much coarser level than we consider here. These results have primarily studied which classes from circuit complexity theory can capture Transformers with different amounts of precision, but these prior results are insensitive to \emph{polynomial blowups} in the amount of precision. (By contrast, our results here will even demonstrate a loss in expressivity from losing a single bit of precision.) For example, \cite{merrill2023logic,merrill2023parallelism,strobl2023average} showed that Transformers using precision logarithmic in the number of tokens can be computed using TC0 circuits\footnote{A TC0 circuit is a Boolean circuit of polynomial size and constant depth consisting of AND, OR, NOT and threshold gates.}. Later work of \cite{chiang2025Transformers} improved this result to more expressive Transformers that use polynomial precision. 
When the precision is restricted to rational numbers, \cite{hao2022formal} showed that even AC0 circuits\footnote{AC0 circuits are like TC0 circuits but without threshold gates.} can simulate hard-attention Transformers, a variant on the usual Transformer. Finally, with  unbounded precision, Transformers can simulate Turing machines \cite{perez2019turing,perez2021attention,merrill2023expressive}.

\paragraph{Equality.} 
Our new fine-grained precision results concern the famous equality function, $\EQ$ which takes as input two strings and should output 1 if and only if the strings are exactly equal (syntactically).

Equality is a basic task that is ubiquitous across a wide variety of applications, and therefore important for Transformers to be able to quickly solve.
For example, it is a basic primitive in databases, e.g., to test whether a particular tracking ID  matches the ID of a product, or whether it appears in a list of IDs. It also the primitive underlying many other tasks such as file synchronization (testing whether two files are identical or whether they need to be synchronized),  user authentication (testing if a proposed password exactly matches a secret string), and automated theorem proving and testing (which requires reasoning about symbolic objects).
{

This equality function has a history of importance in computer science and has been widely studied, including communication complexity \cite{yao1979some,newman1991private}, quantum computing \cite{buhrman2001quantum}, hashing \cite{mccauley2019approximate}, machine learning \cite{singla2006entity,kasai2019low}, etc. } 
{It has also been used to study separations between the expressive power of two-layer Transformers and RNNs by \cite{bhattamishra2024separations}.
} 

\paragraph{One-Layer Transformer Limitations}
In this paper, we focus on the expressive power of single-layer Transformers, which serve as building blocks of multi-layer Transformer architectures. A single layer is a natural unit for understanding token-to-token interactions and serves as one round of attention-based information routing; stacking layers amounts to composing these rounds. Intuitively, if a single layer can compute a key function $f$, then an $L$-layer Transformer could perform more complex tasks which require computing $f$ in sequence $L$ times. 

Representational strengths and limitations of one-layer Transformers are widely studied, and lower bounds have been established for many problems including Parity and Dyck languages \cite{hahn2020theoretical}, Match3 and averaging problems \cite{sanford2023representational}, induction heads task \cite{sanford2024one}, function composition
\cite{peng2024limitations}, counting related tasks \cite{yehudai2024can}, and function evaluation  \cite{strobl2025concise}.

In contrast, very little is proven about multi-layer Transformers; even proving that two-layer Transformers cannot compute key functions appears far beyond current techniques (which are based on communication and circuit complexity approaches, as we discuss more in Section~\ref{sec:main-results} below). The only unconditional lower bound to date was recently established by \cite{chen2024theoretical}, but it is only for decoder-only Transformers, and their proof critically relies on the intricacies of decoder units and do not generalize to Transformers with any encoder units. 

\subsection{Main Results}

We first describe what we mean by a Transformer solving a task. The Transformer can be viewed as a function, which takes $n$ tokens as inputs from the alphabet an $\Sigma$ (input tokens), each token embedded to a vector in $\R^{d}$, where $d$ is the embedding dimension, and outputs $n$ tokens (output tokens). For the inner workings of a Transformer, see \cref{sec:tf}.

In {this paper}, we { study the ability of Transformers} to compute functions of interest. Let us assume that we want to compute a Boolean function $f : \Sigma^n \rightarrow \{0, 1\}$. We consider the Transformer to have $n+1$ input tokens (instead of $n$), where the first $n$ input tokens correspond to the $n$ inputs of the function, and the Transformer is expected to output the value of the function in the $(n+1)$-th (last) output-token.

We prove our tradeoff on a function similar to the equality function. For integers $m>0$ and $n > 2m$, define $\mathcal{T}_{\EQ_m}^n : \{0, 1\}^n\rightarrow \{0, 1\}$, where, given $n$ bits as inputs, $\mathcal{T}_{\EQ_m}^n(x_1, \ldots, x_n)$ will output $1$ if $x_i = x_{m+i}$ for all $i\in [m]$. This checks whether the first $m$ input bits are all equal to the next $m$ input bits (and does not care about the remaining input bits).

As we described before, computing an equality check is a natural and important component of many natural language processing tasks. Our main result shows exactly how many bits of precision are needed to compute it.

\begin{theorem}[Informal]
    \label{thm:tradeoff}
    For every integer $p> 0$ and $n > 4p$, there exists an input representation such that
    \begin{enumerate}
        \item no one-layer Transformer that uses $\leq p-1$ bits of fixed-point precision can solve $\mathcal{T}_{\EQ_{2p-1}}^n$,
        \item there exists a one-layer Transformer that can solve $\mathcal{T}_{\EQ_{2p-1}}^n$ using $p$ bits of fixed-point precision.
    \end{enumerate}
\end{theorem}

This is the tightest tradeoff possible: reducing the precision by just one bit degrades the capability of the Transformer to solve equality. Note that contrary to prior works where precision has been considered as a coarse function of $n$ (either constant, logarithmic in $n$, or polynomial in $n$), we do not place any such restrictions on $p$. The tradeoff works for any value of $p$. 

We also prove similar tradeoff results for the floating-point precision format, where we obtain a tight tradeoff for linear Transformers, and an almost tight tradeoff for soft-max Transformers. The function on which this tradeoff is performed is again very similar to the equality function. However, owing to the mantissa and exponent structure of this precision format, constructing an upper bound (a Transformer that solves this variant of equality task) is more difficult. We discuss the challenges here in more detail in \cref{sec:fp} below and, due to space constraints, the formal results have been moved to \cref{apx:fp}.

These results are primarily theoretical and intended to provide a rigorous framework to study the disadvantages of quantization. This is important as an understanding of expressivity loss can guide the design of quantization schemes that incur as much expressivity loss as the task of interest can tolerate. 
Indeed, it suggests that tasks requiring equality-like comparisons (exact match, membership, etc.) are especially sensitive to quantization. Dropping even one bit can cross a threshold where the model cannot represent the needed comparison reliably. Thus, it paves the way for developing heuristics that will help practitioners choose how much quantization is possible: the precision should be chosen as a function of the length of equality to be checked for the specific task. 

Moreover, this result concretely explains the widely observed phenomenon of empirical loss of expressivity  when quantization is used. As described above, the qualitative difficulty of equality-style tasks under limited precision has been observed in prior works \cite{yao2022zeroquant,wei2023outlier}; our results give a rigorous theoretical explanation for this phenomenon. 

In \cref{sec:exp}, we support the theory with simple experiments on controlled datasets for equality checks, showing that single-layer Transformers with enough precision can learn to compute the equality function, and that decreasing the precision also decreases the accuracy.

\section{Technique Overview}
\label{sec:main-results}

In this section, we give an overview of the main results of precision-expressivity tradeoffs for fixed-point arithmetic.

\paragraph{Input Representations.}
As described before, the function for which we are proving the tradeoff is the Equality function. However, we now discuss  how the inputs to the function are represented as input tokens of the Transformer.

For an integer $n > 0$, a Transformer is a function $\Tf: \Sigma^{n+1}\rightarrow \Sigma^{n+1}_{out}$, that takes as inputs $(n+1)$ input-tokens from the set of alphabets $\Sigma$, and outputs $(n+1)$ output-tokens from the alphabet set $\Sigma_{out}$.

Given a function, $f: \Lambda^N \rightarrow \Sigma_{out}$, a Transformer $\Tf: \Sigma^n \rightarrow \Sigma_{out}^n$ is said to \emph{compute} $f$ for an \emph{input representation} $\mathcal{T} : \Lambda^N\rightarrow \Sigma^n$, for integers $N, n > 0$, if, when the Transformer is given as input $n+1$ tokens, where the $i$-th token is $\mathcal{T}(x_1, \ldots, x_N)_i$ for $i\in [n]$, and the $(n+1)$-th token is $r$ (a placeholder which is a predefined constant and can be considered as zero), outputs the value of $f(x_1, \ldots, x_N)$ in the $(n+1)$-th token of the output. More specifically, 
\begin{equation*}
    \begin{split}
        \Tf(\mathcal{T}(x_1, \ldots, x_N)_1, \ldots,  \mathcal{T}(x_1, \ldots, &x_N)_n, r)_{n+1} \\
         & = f(x_1, \ldots, x_n).
    \end{split}
\end{equation*}

For example, a simple input representation for $n=N$ and $\Lambda = \Sigma$ can be given by $\mathcal{T}(x_1, \ldots, x_n)_i = x_i$ for all $i\in [n]$, i.e., each of the $i$-th input token of the Transformer contains the $i$-th input to the function $x_i$, and the transformer is expected to compute the value of $f(x_1, \ldots, x_N)$.

Using this definition, we will prove our main results. Here, a Transformer is expected to compute the Equality function, $\EQ_m(\y, \z)$, i.e., for given strings $\y, \z \in \{0, 1\}^m$, the Transformer will output $1$ if $\y=\z$, otherwise $0$. The input representation will be similar to the simple one described above, where the first $m$ tokens will contain one of each bit of $\y$, followed by the next $m$ tokens which will contain one of each bit of $\z$, followed by placeholder zeros to fill up the number of input tokens of the Transformer to $n+1$.

In all the results, the underlying function that the Transformer computes is always Equality $\EQ_m$. However, we modify the input representation somewhat 
so that testing Equality of two strings can be checked in the required precision format, and to prove lower and upper bounds.

\paragraph{Lower Bounds.} We prove the lower bounds, i.e., Transformers cannot compute the Equality function, using tools from communication complexity. Similar techniques of proving lower bounds for Transformers have been studied before (including works of \cite{peng2024limitations,sanford2023representational,sanford2024one,yehudai2024can,chen2024theoretical,chakrabarti2026poly}), but all of them treat precision as a technical artifact rather than something important. They give lower bounds on the size of the Transformer and present their results as: $Hdp$ should be lower-bounded by some function of $n$. Recently, \cite{kozachinskiy2025strassen} removed this dependence on precision, $p$, using a concept called Split VC-dimension. 

However, we demonstrate how  this dependence on precision should not be ignored, and it is fundamental in understanding the expressivity of Transformers.

\paragraph{Upper Bounds.} For the upper bounds, we give an explicit construction of a Transformer to show that Transformers can compute the Equality function (check if $\y = \z$) using query dimension $d_q = 3$ and value dimension $d_v = 1$. This construction has the basic idea: the numerator of the softmax self-attention is used to check if the first halves of the two strings are equal, and the denominator is used to check if the second halves are equal. This is performed by choosing the embedding function such that the value matrices of the tokens from the first halves of $\y, \z$ contribute to the numerator, while the remaining are zeros. For the denominator, the embedding function makes query and key matrices are such that the values of $e^{\langle Q_i, K_j\rangle}$ are small, and ignored for the first half of the strings, while the only contribution to the value of the denominator is from the tokens corresponding to the second halves of $\y, \z$.

We use careful low-precision arithmetic to prevent information getting lost due to rounding. An overview of the proofs have been given in the following subsections, and the complete proofs can be found in Appendix \ref{apx:proof-fx}.

\subsection{Tradeoff with Fixed-Point Precision: Simple Equality Function}
\label{sec:EQ}

First, we show that using a simple input representation, $\EQ_m$ can be computed by a Transformer using $p$ bits, but cannot be computed using $p-2$ bits, where $m := 2p-3$.

Formally, we start with two strings $\y, \z \in \{0, 1\}^m$, and are interested in computing the function $\EQ_m(\y, \z)$ with the promise conditions:
    \begin{itemize}
    \item $m$ is odd.
    \item The binary values satisfy $\y \leq \z$.
    \end{itemize}

The input representation $\Tm^0 : \{0, 1\}^{2m} \rightarrow \{0, 1\}^{n}$ is as follows:
\begin{enumerate}
        \item $\Tm^0(\y, \z)_{i} = y_{i}$ for $i \in [m]$.
        \item $\Tm^0(\y, \z)_{i+m} = z_{i}$ for $i\in [m]$.
        \item $\Tm^0(\y, \z)_{i} = 0$ for all $i \in  [2m+1: n]$. 
    \end{enumerate}

    This input representation essentially implies that the first $m$ tokens contain each bit of $\y$, the next $m$ tokens contain each bit of $\z$, and the rest are zeros.

    The output of the Transformer (the value at the last output token) will be: 
    \begin{itemize}
        \item $1$ if $\y=\z$ and $m$ is odd,
        \item $0$ if $\y \neq \z$ and the promise conditions are met, and
        \item don't care otherwise.
    \end{itemize}
    
    With this representation, we prove the following result:

\begin{theorem}[Tradeoff with Fixed-Point Precision for Equality]
    \label{thm:tradeoff-eq}
    For every integer $p> 0$ and every integer $n > 4p$, choosing $m = 2p-3$,
    \begin{enumerate}
        \item no one-layer Transformer with $n$ tokens having value dimension $d_v=1$ and one head, can compute $\EQ_m$ with the input representation $\Tm^0$, using $ p\!-\!2$ bits of precision, 
        \item there exists a one-layer Transformer with $n$ tokens having value dimension $d_v = 1$ and one head, that can compute ${\EQ_m}$ with the input representation $\Tm^0$ using $p$ bits of precision in the fixed-point precision format. 
    \end{enumerate}
\end{theorem}

Even though this is not a tight tradeoff, this is enough to show a separation in real life quantization regimes for fixed-point precision, where computations are usually often on $4, 8, 16$ bits. However, this gap can be bridged to make the tradeoff tight, as will be shown in the next section.

    The proof of 
Theorem \ref{thm:tradeoff-eq} will follow from the following two lemmas. 
Lemma 2.2 
states the lower bound, showing that for our choice of parameter settings ($m=2p-3$), no single layer Transformer with $p-2$ bits of precision can compute $\EQ_m$.  Lemma 2.3 
states the near matching upper bound, showing that $\EQ_m$ can be computed using $p+1$ bits of precision. 
(Note that since we have set $m=2p-3$ and $m$ is odd, $\lceil m/2 \rceil -1 =p-2$ and  $\lceil m/2 \rceil +1 = p$.)

\begin{restatable}[Fixed-Point Lower Bound]{lemma}{LBEQ}
\label{lem:lb-eq}
    Given the input representation $\Tm^0$, any one-layer Transformer with $d_v=1$ and $H=1$, cannot compute $\EQ_m$ with $\leq \mm -1 $ bits of precision.
\end{restatable}

\begin{restatable}[Fixed-Point Upper bound]{lemma}{UBEQ}
\label{lem:ub-fixed-pt-eq}
    There exists a one-layer Transformer with $d_v=1$ and $H=1$, that uses $\mm+1$ bits of fixed-point precision to compute $\EQ_m$ in the input representation $\Tm^0$.
\end{restatable}

\cref{lem:lb-eq} 
is proven via a reduction to a communication complexity lower bound, similar at a high level to other communication complexity based lower bound proofs for Transformers
\cite{peng2024limitations,sanford2023representational}.
Next we give a proof overview for \cref{lem:ub-fixed-pt-eq}. 
The full proofs of both Lemmas can be found in \cref{apx:proof-fx-eq}.

\begin{proof}[Proof overview.] The main idea of this construction is inspired by the proof of the lower bound, where both the numerator and denominator of self-attention are used to transfer information. The construction uses the numerator to check equality in the first half (whether $\y_{1:\mm} = \z_{1:\mm}$), and then use the denominator to check the rest. 

For the numerator part, we will choose the embedding such that the final numerator value is equal to the maximum possible value in this representation, $2^{\mm}-1$, if they are equal. We ensure that by making the summation value is $$\sum_{i=1}^{n+1}e^{\langle Q_{n+1}, K_i\rangle}V_i = -\y_{1:\mm} + \z_{1:\mm} + \underbrace{1\ldots 1}_{\mm}.$$

In order to obtain this value, we design the embedding and the $W^Q, W^K, W^V$ matrices in a way such that $\sum_{i = 1}^{\mm} e^{\langle Q_{n+1}, K_i\rangle}V_i$ gives the value $-\y_{1:\mm}$, and $\sum_{i=m+1}^{m+\mm} e^{\langle Q_{n+1}, K_i\rangle}V_i$ gives the value $\z_{1:\mm}$. Along with this, there will be a dummy token in the $(2m+1)$-th index such that $e^{\langle Q_{n+1}, K_{2m+1}\rangle}V_{2m+1} = 2^{\mm}-1$, and the rest will be zeros. The value of the numerator will be exactly $(2^{\mm}-1)$ if $\y_{1:\mm}=\z_{1:\mm}$, otherwise infinity (commonly used as \textsf{float(`inf')} in PyTorch), which follows from the promise on $\y, \z$.

For the denominator, we enforce $\sum_{i=1}^{\mm}e^{\langle Q_{n+1}, K_{j}\rangle}$ and $\sum_{i=m+1}^{m+\mm} e^{\langle Q_{n+1}, K_{j}\rangle}$ are small constants irrespective of $\y, \z$, and summation $\sum_{i=\mm+1}^{m}e^{\langle Q_{n+1}, K_{j}\rangle} + \sum_{i=m+\mm+1}^{2m}e^{\langle Q_{n+1}, K_{j}\rangle}$ will be $\bar \y_{\mm+1:m} +  \z_{\mm+1:m} + 2^{-1}$. This is a fixed value if $\y_{\mm+1:m} = \z_{\mm+1:m}$, otherwise a larger value.

Summing up, if $\y_{1:m} = \z_{1:m}$, then the numerator will be $(2^{\mm}-1)$ and the denominator will be a fixed value $2^{(\mm-1)}-1 + 2^{-1}$, otherwise either the numerator make the output infinity (\textsf{float(`inf')}), or the denominator will a larger value (therefore the output different from $\frac{2^{\mm}-1}{2^{\mm-1}-2^{-1}} = 2$), either of which can be detected with a simple MLP layer after the self-attention computation.
\end{proof}

\subsection{Tradeoff with Fixed-Point Precision: Theorem \ref{thm:tradeoff}}

In this section, we prove a tight one-bit tradeoff. We use the same function $\EQ_m$ as before, along with the same promise that $m$ is odd and $\y \leq \z$.

    For ease of notation, we  consider indexing from $-1$ up  $n-2$. The input representation $\Tm^1 : \{0, 1\}^{2m}\rightarrow \{0, 1\}^n$ is defined as follows:
        \begin{enumerate}
        \item $\Tm^1(\y, \z)_{-1} = 0$.
        \item $\Tm^1(\y, \z)_0 = y_1$.
        \item $\Tm^1(\y, \z)_{i} = y_{i}$ for $i \in [m]$.
        \item $\Tm^1(\y, \z)_{i} = z_1$ for $i \in \{m+1, m+2\}$.
        \item $\Tm^1(\y, \z)_{m+3} = 0$.
        \item $\Tm^1(\y, \z)_{i} = z_{i-m-2}$ for $i\in [m+4: 2m+2]$.
        \item $\Tm^1(\y, \z)_{i} = 0$ for all $i \in  [2m+3: n-2]$. 
    \end{enumerate}

Formally, Theorem \ref{thm:tradeoff} can be stated as follows.

\begin{theorem}[Tight Tradeoff with Fixed-Point Precision]
    \label{thm:tradeoff-main}
    For every integer $p> 0$ and every integer $n > 4p$, choosing $m = 2p-1$,
    \begin{enumerate}
        \item no one-layer Transformer with $n$ tokens having value dimension $d_v=1$ and one head, can compute $\EQ_m$ with the input representation $\Tm^1$, using $ p-1$ bits of precision, 
        \item there exists a one-layer Transformer with $n$ tokens having value dimension $d_v = 1$ and one head, that can compute ${\EQ_m}$ with the input representation $\Tm^1$ using $p$ bits of precision in the fixed-point precision format. 
    \end{enumerate}
\end{theorem}

As we will see in the proof, padding tokens do increase the expressivity of the Transformer and helps us achieve this exact trade off for this function. On a theoretical standpoint, using padding tokens do indeed improve the expressivity of Transformers was also proved by \cite{merrill2025exact}. In our case, we use this representation to prevent overflow as we are dealing with the exact number of bits that can be used to compute the function. We first state the lower bound present in \cref{apx:rem-proofs}.

\begin{restatable}[Lower Bound]{lemma}{LBEQFX}
\label{lem:lb}
    Given the input representation $\Tm^1$, any one-layer Transformer with $d_v=1$ and $H=1$ that computes $\EQ_m$ requires precision $ \geq \mm$ bits.
\end{restatable}

    Now, we prove the existence of a Transformer with precision $\mm$ bits which actually computes $T_m$.

\begin{restatable}[Upper Bound]{lemma}{UBEQFX}
\label{lem:ub-fixed-pt}
    There exists a one-layer Transformer with $d_v=1$ and $H=1$, that uses $\mm$ bits of fixed-point precision to compute $\EQ_m$ in the input representation $\Tm^1$.
\end{restatable}

\begin{proof}[Proof overview.] The construction is similar to the proof of \cref{lem:ub-fixed-pt-eq}, where the numerator of the self-attention checks equality in the first half (whether $\y_{1:\mm} = \z_{1:\mm}$), and then the denominator checks the rest.

For the numerator part, we will choose the embedding such that the final numerator value is equal to one if they are equal. We make the summation value $\sum_{i=-1}^{n+1}e^{\langle Q_{n+1}, K_i\rangle}V_i = \y_{1:\mm} - \z_{1:\mm} + 1$, by ensuring $\sum_{i = -1}^{\mm-1} e^{\langle Q_{n+1}, K_i\rangle}V_i$ gives $\y_{1:\mm}$, and $\sum_{i=m}^{m+\mm} e^{\langle Q_{n+1}, K_i\rangle}V_i$ gives $\z_{1:\mm}$. However, this is not directly possible as before, since $\y_{1:\mm}, \z_{1:\mm}$ use $\mm$ bits, but in $\mm$ bits of fixed-point precision, one bit is allocated to the sign bit, giving us $\mm-1$ positions. 
To circumvent this difficulty, we use  a dummy token in the $(-1)$-th position, such that $e^{\langle Q_{n+1}, K_{-1}\rangle}V_{-1} = -(2^{\mm-1}-1)$, and is used to add a negative shift to the value of $\y_{1:\mm}$ in order to bring it to the permissible range of $\mm$ precision bits.

There is one more difficulty, the value $\y_12^{\mm-1}$ might be infinity while computing $\y_{1:\mm}$. In order to prevent this, we split this value amongst two tokens corresponding to $0$ and $1$ positions, where each of them has the value $e^{\langle Q_{n+1}, K_i\rangle}V_i = \y_12^{\mm-2}$. This will give the first part of the summation of the numerator as $$\sum_{i\in [-1: m]} e^{\langle Q_{n+1}, K_i\rangle}V_i = \y_{1, \mm} - \underbrace{1\ldots 1}_{\mm-1}.$$

A similar shift is added to $-\z_{1:\mm}$, to obtain:
$$\sum_{i\in [m+1: 2m+2]} e^{\langle Q_{n+1}, K_i\rangle}V_i = -\z_{1, \mm} + \underbrace{1\ldots 1}_{\mm-1}.$$

Finally, we  choose the embedding such that $e^{\langle Q_{n+1}, K_{2m+3}\rangle}V_{2m+3} = 1$, such that the value of the numerator will be exactly $1$ if $\y_{1:\mm}=\z_{1:\mm}$, otherwise $\leq 0$ from the promise on $\y$ and $\z$.

For the denominator, again we ensure that $\sum_{i=1}^{\mm}e^{\langle Q_{n+1}, K_{j}\rangle}$ and $\sum_{i=m+1}^{m+\mm} e^{\langle Q_{n+1}, K_{j}\rangle}$ are small constants irrespective of $\y, \z$, and ignored with respect to the summation $\sum_{i=\mm+1}^{m}e^{\langle Q_{n+1}, K_{j}\rangle} + \sum_{i=m+\mm+3}^{2m+2}e^{\langle Q_{n+1}, K_{j}\rangle}$. This gives the value $\bar \y_{\mm+1:m} + \z_{\mm+1:m}$, which is a fixed value for $\phi_{\mm+1:m} = \phi_{m+\mm+1:2m}$, otherwise a larger than that due to the promise.

Summing up, if $\y=\z$, then the numerator will be $1$ and the denominator will be a fixed value $2^{(\mm-1)}-1$, otherwise either the numerator will be non-positive or the denominator a larger value (i.e., \textsf{float(`inf')}), either of which can be detected by the final MLP layer.
\end{proof}

\begin{proof}[Proof of Theorem \ref{thm:tradeoff}]
    We choose $m := 2p-1$. From Lemma \ref{lem:lb}, $\EQ_m$ cannot be computed by a Transformer using $\leq \mm-1 = p-1$ bits for input representation $\Tm^1$, but can be computed by a Transformer using $\mm = p$ bits (Lemma \ref{lem:ub-fixed-pt}).
\end{proof}

\subsection{Overview of Floating-Point Results}
\label{sec:fp}

We prove similar bounds for floating-point precision format with the same function $\EQ_m$. However, the input representation will be different-- instead of each token containing a bit of $\y$ or $\z$, we will allow each token to contain a tuples, where one-element is a tuple of $e-1$ bits from either $\y$ or $\z$ (which will be used to directly create the exponent in the floating-point precision format for the Transformer construction in the upper bound), and one more bit (which will contribute to the mantissa). 

This difficulty arises because, in the floating-point representation, the exponent will be exponentiated twice for computing $e^{\langle Q_i, K_j\rangle}$, and summation over this to yield another floating-point number is difficult as it might lead to overflow in bounded precision format.

 This choice of input tokens containing tuples is a fairly known technique, and have been used in other Transformer based problems, including function evaluation where the indices and values are given as tuples \cite{strobl2025concise},  graph problems where the edges are given as tuples \cite{sanford2024understanding}, higher-order graph Transformers where $k$-tuples of vertices are used \cite{zhou2024theoretical} etc. Indeed, reformatting of inputs is an important technique that can completely alter what a Transformer can learn or represent \cite{zhou2023algorithms}. More details and the complete proofs can be found in \cref{apx:fp}.

\section{Experiments}
\label{sec:exp}

In this section, we experimentally verify the theoretical results of this paper. We conduct experiments on the ability of a quantized one-layer transformer to check Equality. This has two main goals: to show that weights allowing such a transformer to compute equality can be learned, and to demonstrate the empirical degradation of accuracy when our results show that the precision is too low to test for equality.

\paragraph{Experimental Setup.}
A simple input representation is used, where the first $m$ tokens contain each of the corresponding $m$ bits of the first string, the next $m$ tokens contain each of the corresponding $m$ bits of the second string, and the last token contains zero and is the placeholder token where the output is expected.


A one-layer transformer is trained with number of heads $H = 2$ and embedding dimension $d = 4$. Other details of the input generation, transformer architecture, training dynamics, quantization procedures, etc., can be found in \cref{apx:exp_eq}. The experiments were repeated over 10 random seeds.

\paragraph{Experimental Results.} The communication based lower bounds, presented in this paper and the ones introduced by \cite{peng2024limitations}, state that if $H(d+1)p < m$, where $H$ is the number of heads, $d$ is the embedding dimension, and $p$ is the number of bits of precision; then the transformer will not be able to solve equality. We verify this experimentally, where we show that if $H(d+1)p$ is close to $m$, then the transformer does not perform well with the equality task.

The model was initially trained on $\fb$ precision, and the accuracies and standard deviations obtained for $m=100$ was $100.00\pm 0.00$, for $m=50$ was $100.00\pm 0.00$, for $m=30$ was $99.62\pm 0.12$, and for $m=15$ was $99.96\pm 0.12$. The tables report the accuracies and standard deviations for post-training quantization (PTQ) with fixed-point precision formats: INT12, INT8, INT6, INT4, followed by those after quantization aware training (QAT).


\begin{table}[H]
\centering
    \resizebox{\columnwidth}{!}{%
    \begin{tabular}{c|cccc}
        $m$ & INT12 &  INT8 & INT6 & INT4 \\\hline
        $100$	& $94.95 \pm 15.98$	& $94.21 \pm 12.21$	& $71.24 \pm 20.52$	& $58.72 \pm 16.58$\\
$50$	& $97.43 \pm 8.14$	& $99.47 \pm 1.57$	& $88.53 \pm 11.58$	& $74.20 \pm 16.86$\\
$30$	& $99.70 \pm 0.96$	& $98.20 \pm 5.67$	& $85.03 \pm 19.01$	& $64.56 \pm 16.06$\\
$15$	& $99.99 \pm 0.04$	& $100.00 \pm 0.00$	& $86.49 \pm 20.30$	& $72.44 \pm 16.27$
    \end{tabular}
    }
    \caption{Acc.$\pm$SD $(\%)$ after PTQ (fixed-point precision).}
    \label{tab:fx}
\end{table}



\begin{table}[H]
\centering
    \resizebox{\columnwidth}{!}{%
    \begin{tabular}{c|cccc}
        $m$ & INT12 &  INT8 & INT6 & INT4 \\\hline
        $100$ & $94.02 \pm 15.39$ & $92.71 \pm 12.73$ & $68.75 \pm 18.15$ & $59.67 \pm 16.34$\\
        $50$ & $100.00 \pm 0.00$ & $99.75 \pm 0.58$ & $87.99 \pm 18.23$ & $75.76 \pm 19.72$\\
        $30$ & $98.26 \pm 5.50$ & $96.07 \pm 12.41$ & $92.68 \pm 14.73$ & $69.10 \pm 20.60$ \\
        $15$ & $100.00 \pm 0.00$ & $99.88 \pm 0.37$ & $93.59 \pm 8.53$ & $78.34 \pm 17.17$
    \end{tabular}
    }
    \caption{Acc.$\pm$SD $(\%)$ after QAT (fixed-point precision).}
    \label{tab:fx_qat}
\end{table}

Next, we show the loss in accuracy for quantization in floating-point format for both PTQ and after QAT. The same base model from the previous section used, and then we perform quantization to FP16 (5 exponent, 10 mantissa and one sign bit), FP8\_E5M2 (5 exponent bits, 2 mantissa bits and one sign bit) as well as FP8\_E4M3 (4 exponent bits, 3 mantissa bits and one sign bit), which are floating-point precision formats pertaining to the IEEE 754 standards and \cite{micikevicius2022fp8}. The tuples $(a, b)$ in the headings of the tables refer to $a$ bits of exponent and $b$ bits of mantissa, with $\fb$ being a default of $(8, 23)$.

\begin{table}[H]
\centering
    \resizebox{\columnwidth}{!}{%
    \begin{tabular}{c|ccc}
        $m$ &  $(5, 10)$& $(5, 2)$ & $(4, 3)$\\\hline
        $100$  & $91.81 \pm 17.80$ & $60.48 \pm 18.20$ & $70.94 \pm 25.07$\\
        $50$  & $95.02 \pm 15.74$ & $59.25 \pm 17.74$ & $87.55 \pm 14.39$\\
        $30$  & $99.73 \pm 0.85$ & $74.95 \pm 21.05$ & $79.26 \pm 19.99$ \\
        $15$  & $100.00 \pm 0.00$ & $73.36 \pm 17.65$ & $85.70 \pm 17.04$
    \end{tabular}
    }
    \caption{Acc.$\pm$SD $(\%)$ after PTQ  (floating-point precision).}
    \label{tab:eq}
\end{table}



\begin{table}[H]
\centering
    \resizebox{\columnwidth}{!}{%
    \begin{tabular}{c|ccc}
        $m$ &   $(10, 5)$& $(5, 2)$ & $(4, 3)$\\\hline
        100 &  $95.27 \pm 14.95$ & $65.25 \pm 21.84$ & $76.04 \pm 22.73$\\
        50 & $97.75 \pm 7.10$ & $65.06 \pm 23.17$ & $85.39 \pm 20.26$\\
        30 & $98.61 \pm 4.39$ & $74.61 \pm 14.58$ & $86.95 \pm 18.20$\\
        15 & $100.00 \pm 0.00$ & $75.06 \pm 14.99$ & $90.39 \pm 10.76$
    \end{tabular}
    }
    \caption{Acc.$\pm$SD $(\%)$ after QAT (floating-point precision).}
    \label{tab:eq_ft}
\end{table}

Note that having an accuracy $50\%$ is randomly outputting a zero or one. However, these experiments are in sync with our theory results-- INT12 (where $H(d+1)p = 120$) always performs well since all the $m\leq 100$, INT8 (which has $H(d+1)p = 80$) performs slightly worse for $m=100$ but almost perfectly for $m \leq 50$, INT6 (which has $H(d+1)p = 60$) performs quite poorly for $m=100$ but fairly well with around $85\%$ accuracy for $m\leq 50$, INT4 (which has $H(d+1)p = 40$) performs fairly well for $m=15$ but poorly for $m\geq 30$.

Similarly, for fixed-point precision, FP16 (with $H(d+1)p = 160$) performs fairly well in all the tasks, FP8 (with $H(d+1)p = 80$), performs poorly for $m=100$, but fairly well for $m\leq 50$ for FP8\_E4M3, although performance degrades with FP8\_E5M2.

These results verify that equality is difficult to quantize, and formalizes the evidence of accuracy loss with quantization on equality-related tasks in real-life.

\section{Conclusion}

This paper theoretically justifies that expressivity loss is inevitable with quantization. It further aims to develop a class of tasks on which quantization is difficult-- practitioners need to be careful while choosing quantization schemes on these. 

The precision-expressivity trade-offs proved here can be extended beyond $H, d_v =1$, by dividing the strings $\y, \z$ into $H(d_v+1)$ parts, and computing $\EQ(\y, \z)$ in parallel. However, finding such a trade-off for the multi-layer Transformer remains elusive due to the lack of lower-bound techniques. Overcoming this challenge will help design quantization schemes where accuracy will be maintained.




\bibliography{references}
\bibliographystyle{icml2026}

\newpage

\appendix
\onecolumn

\section{Notations and Preliminaries}
\label{apx:prelims}

For integers $n, m$, we use the notation $[n]$ to denote the set $\{1, 2, \ldots, n\}$, $[m: n]$ to denote the set of integers $\{m, m+1, \ldots, n\}$, where $m \leq n$, and $(m, n)$ (resp.~$[m, n]$) to denote the set of real numbers between $m$ and $n$, end points excluded (resp.~end points included).  For a matrix ${A}\in \R^{n\times m}$ and for $i\in [n]$, ${A}_i$ denotes the $i$-th row vector $({A}_{i, 1}, \ldots, {A}_{i, m})\in\R^{m}$. Similarly, for $j\in [m]$, ${A}_{:, j}$ denotes the column vector $({A}_{1, j}, \ldots, {A}_{n, j})\in\R^{n}$.

Strings are emboldened, and for a string $\mathbf{a} \in \{0, 1\}^n$, $\mathbf{a}$ will denote the value of the number whose binary expression is $a_1a_2\ldots a_n$, and $\mathbf{a}_{p:q}$ represents the substring $a_pa_{p+1}\ldots a_q$, for $1 \leq p < q \leq n$.

\subsection{Transformers}
\label{sec:tf}


We define the multi-layer Transformer architecture as follows.
\begin{definition}
    An $L$-layer Transformer is a function $\Sigma^n \rightarrow \Sigma^n$, where $\Sigma$ is an alphabet set and $n$ is the number of tokens, also regarded as the sequence length. It consists of:
    \begin{enumerate}
        \item an embedding function $\mathbb{E} : \Sigma \rightarrow \R^{d}$, where $d$ is the embedding dimension,
        \item several Transformer units for each layer that are functions $ \mathsf{Tf}_i:\R^{n\times d} \rightarrow \R^{n\times d}$ for $i\in [L]$,
        \item and a final output projection function, also known as language modeling head, $f_{LM} : \R^d \rightarrow [0, 1]^{|\Sigma|}$, which give the probabilities of every alphabet from the alphabet set $\Sigma$.
    \end{enumerate}  
    Each Transformer layer consists of two components, a self-attention function $f^i_{SA}:\R^{n\times d}\rightarrow \R^{n\times d_v}$ and an MLP function $f^i_{MLP}: \R^{d_v} \rightarrow \R^d$, where $d_v$ is the value dimension. Now, an $L$-layer Transformer is the composition of the functions:
    \begin{equation*}
        f_{LM}\circ f_{MLP}^L \circ f^L_{SA}\circ \ldots \circ f_{MLP}^1 \circ f_{SA}^1 \circ \mathbb{E} : \rightarrow \Sigma^{n}\rightarrow [0, 1]^{|\Sigma|}.
    \end{equation*}
\end{definition}



The core operation in a Transformer layer is the self-attention function, which is used to compute cross-token interactions. Soft-max self-attention is the most commonly used form of self-attention and is defined as follows.

\begin{definition}[Softmax Self-Attention]
\label{def:soft-SA}
Given parameters $W^Q, W^K \in \R^{d\times d_q}$, and $W^V \in \R^{d\times d_v}$, where $d_q$ is the query dimension and $d_v$ is the value dimension, the soft-max self-attention is a function $\R^{n\times d} \rightarrow \R^{n\times d}$, where the $i$-th row of the output is computed as
    \begin{equation*}
        SA_i = \frac{\sum_{j\in [n]}e^{(X_i W^Q (W^K)^T X_j^T)} X_jW^V}{\sum_{j\in [n]}e^{(X_i W^Q (W^K)^T X_j^T)}},
    \end{equation*}
    for all $i\in [n]$.
\end{definition}

We will commonly refer to $XW^Q$ as the query matrix $Q$, $XW^K$ as the key matrix $K$ and $XW^V$ as the value matrix $V$. Also, $\sum_{j\in [n]}e^{(X_i W^Q (W^K)^T X_j^T)} X_jW^V$ will be called the numerator term and $\sum_{j\in [n]}e^{(X_i W^Q (W^K)^T X_j^T)}$ the denominator term.

We also define the linear self-attention mechanism, which is widely used as a more computationally efficient counterpart of the soft-max self-attention mechanism.

\begin{definition}[Linear Self-Attention]
\label{def:linear-att}
    Given parameters $W^Q, W^K \in \R^{d\times d_q}$, and $W^V \in \R^{d\times d_v}$, the linear self-attention is a function $\R^{n\times d} \rightarrow \R^{n\times d}$, where the $i$-th row of the output is computed as
    \begin{equation*}
        SA_i = \sum_{j\in [n]}X_i W^Q (W^K)^TX_j^T X_j W^V ,
    \end{equation*}
    for $i\in [n]$.
\end{definition}


In our constructions, given tokens $X_1, \ldots, X_n$ we will consider them to contain positional encodings, and the embedding as a function $[n]\times \Sigma \rightarrow \R^d$, where the inputs to the function are the index of the token and the value in $\Sigma$, and the output $f(i, X_i)$ is a vector in $\R^d$. 


\begin{definition}[Final MLP Layer]
    The final MLP layer is a function $f_{MLP}: \R^{d_v} \rightarrow \R^d$, which has $d_v$ input nodes, $d$ output nodes, and $d_{ff}$ nodes in one hidden layer in between them (usually $d_{ff}= 4d$). Each edge from nodes one layer to the next is associated with weights and biases. 
    
    The weights of the $i$-th input node to the $j$-th hidden node is denoted as $w_{i, j}^1$, and the bias of the $j$-th hidden node as $b^1_j$, for $i\in [d_v], j\in [d_{ff}]$. Subsequently, the weight of the $j$-th hidden node to the $k$-th output node is denoted as $w_{j, k}^2$, and the bias of the $k$-th output node as $b^2_k$, for $j\in [d_{ff}], k\in [d]$.

    Given the values of the hidden nodes as $v_1, \ldots, v_{d_v}$, the $j$-th hidden node will have the value $$h_j = \max\{0, \sum_{i\in [d_v]}w^{1}_{i, j}v_i + b_j\},$$
    for $j\in [d_{ff}]$, and the $k$-th output node will have the value
    $$\max\{0, \sum_{j\in [d_{ff}]}w^{2}_{j, k}h_j + b_k\},$$
    for $k\in [d]$.
\end{definition}

\subsection{Communication Complexity}

We use a model of communication with two players Alice and Bob, each of whom have unbounded computational power. In order to solve the problem, Alice has some bits, Bob has some bits and they want to communicate to compute a function known to all. Furthermore, we only allow one-way communication which is communication only from Alice to Bob. Interested readers can see \cite{kushilevitz1997communication} for more background on communication complexity.

\begin{definition}[Fooling sets]
    Given a function $f$, a set $F = \{(x_i, y_i)\, | \, i \in [\ell]\}$ is called a fooling set for $f$ of size $\ell$ if:
    \begin{itemize}
        \item for all $i\in \ell$, $f(x_i, y_i)$ has the same value,
        \item for all $i\neq j$, $i, j\in [\ell]$, either $f(x_i, y_j)$ or $f(x_j, y_i)$ has a value different from $f(x_i, u_i)$.
    \end{itemize}
\end{definition}

\begin{lemma}[Folklore]
    \label{lem:fooling-set}
    If the function $f$ has a fooling set of size $\ell$, then the communication complexity of $f$ is at least $\lceil \log_2 \ell \rceil$.
\end{lemma}

\subsection{Computations with Bounded Precision}

In this article, we use a certain kind of precision defined as follows.
\begin{definition}[Fixed-point precision]
\label{def:fixed-p}
    We define the \emph{fixed-point precision} by a triple $\langle s, a, b\rangle$, where $a, b$ are binary expression of integers and $s$ is the sign bit. The value is given by 
    \begin{equation*}
        \langle 0, a, b \rangle = a + 2^{-|b|}b = a \, \cdot \, b\; ,
    \end{equation*}
    and 
    \begin{equation*}
        \langle 1, a, b \rangle = -(a + 2^{-|b|}b) = -a \, \cdot \, b\; .
    \end{equation*}
    We say that this requires $p$ bits of precision if $1+|a| + |b| = p$. 
\end{definition}

Typically, fixed-point precision formats also have a scaling factor, say $\lambda$, associated with them, which brings the final value in the range $[-\lambda(2^{p-1}-1): \lambda (2^{p-1}-1)]$

For addition or multiplication of two numbers, we first find the exact value of the two numbers, and then round off the answer to bring it back to the given precision bounds. Approximation occurs by rounding-off to the nearest number within the permissible set of numbers. Ties can be broken by truncation.

\begin{example}
    To compute the sum of $1\cdot1101$ and $10\cdot111$ using $p=5$, the answer becomes 
    \begin{equation*}
        1\cdot1101 + 10\cdot111 = 100\cdot0011 = 100\cdot 01.
    \end{equation*}
    Similarly, for multiplying $1\cdot 101$ with $10\cdot 11$ using $p=4$, we get
    \begin{equation*}
        1\cdot 101\times 10\cdot 11 = 100\cdot 01111 = 100\cdot 1.
    \end{equation*}
\end{example}

Note that in this model, $|a|, |b|$ do not need to be fixed separately, which gives it more representational power. When we perform iterated addition or multiplication, we use \emph{left associativity}, i.e., given $a+b+c$ (resp.~$a\times b\times c$), we compute it by first computing $(a +b)$ (resp.~$a\times b$), find the approximation to the $p$ most significant bits, and then add (resp.~multiply) that value to $c$, followed by another approximation to $p$ most significant bits. For summing $n$ elements, we iteratively perform this as $$a_1+a_2+\ldots +a_n = ((((a_1+a_2)+a_3) + \ldots) +a_n),$$
and multiplication is also 
$$a_1\times a_2\times \ldots\times a_n = ((((a_1\times a_2)\times a_3) \times \ldots) \times a_n).$$

\begin{example}
    We show why left associativity is important, as consider $10\cdot0$, $1\cdot01$, $1\cdot11$ with $p = 3$. Now, $((10\cdot0+1\cdot01)+1\cdot11) = (11+1\cdot11) = 100$, while $(10\cdot0+(1\cdot01+1\cdot11)) = (10\cdot0+11\cdot0) = 101$.
\end{example}

Similarly, for multiplication, given $a, b$, in order to compute $a\times b$, we first compute the exact value, and then consider the $p$ significant bits in this precision format.

This precision format is directly inspired from communication complexity. Our lower bounds are in terms of Alice sending bounded precision numbers to Bob. For this fixed-point precision format, Alice sending each number to Bob is equivalent to her sending 3 consecutive messages to Bob, the first being the sign bit, then the number before the decimal, and  finally  the number after the decimal.

We also define the floating-point precision format as follows.

\begin{definition}[Floating-point precision]
    \label{def:float-p}
    We define the \emph{floating-point precision} by a quadruple $\langle s_a, a, s_b, b\rangle$, where $a, b$ are binary expression of integers and $s_a, s_b$ are the corresponding sign bits. The value represented by this is given by 
    \begin{equation*}
        \langle s_a, a,s_b, b \rangle = (-1)^{s_a}\times 1\cdot a \times 2^{(-1)^{s_b}\times b}.
    \end{equation*}
    We say that this requires $p$ bits of precision if $1+|a| + 1+ |b| = p$. 
\end{definition}

Here, $1\cdot a$ is called the \emph{mantissa} and $b$ is called the \emph{exponent}.



Again, addition and multiplication of two numbers are first computed exactly, and then rounded off to the desired format of the floating point precision. Both addition and multiplication of more than two numbers follow left-associativity.

\subsection{Self-attention Computation with Bounded Precision}
\label{apx:comp}

In the bounded precision, we consider self-attention to be computed using left-associativity for both addition and multiplication. Denoting $Q:= XW^Q$ as the query matrix and $K := XW^K$ as the key matrix, the output of self-attention can be seen as the value 
\begin{equation*}
    D^{-1} A XW^V,
\end{equation*}
where $A = e^{QK^T} \in \R^{n\times n}$, $D$ is a diagonal matrix in $\R^{n\times n}$ such that $D_{i, i} = A_i \mathbf{1}_n$. The computation is performed as:
\begin{enumerate}
    \item First compute $Q := XW^Q$, $K = XW^K$ in bounded precision regime. For $i\in [n], j\in [d_q]$, this is performed as:
    \begin{equation*}
        Q_{i, j} = \sum_{k\in [d]} X_{i, k}W^{Q}_{k, j},
    \end{equation*}
    where each of $X_{i, k}W^{Q}_{k, j}$ is first computed in bounded precision regime, and then the summation is performed using left associativity. A similar thing is done to compute $K$.
    \item Similarly, we compute $QK^T \in \R^{n\times n}$.
    \item Subsequently, we compute $A := e^{QK^T} \in \R^{n\times n}$.
    \item After this, we compute the normalizing term $D$ from $A$.
    \item We then matrix multiply as:
    \begin{equation*}
        D^{-1}A XW^V = D^{-1}((AX)W^V),
    \end{equation*}
    { i.e., first compute $A$, store it and then compute $AX$, followed by computing $(AX)W^V$ which gives the numerator of the softmax, followed by computing $D$ from $A$} to give the final output matrix. Here, $((AX)W^V)$ is the numerator term and $D$ is the denominator term. 
\end{enumerate}

\section{Proofs of Section \ref{sec:main-results}}
\label{apx:proof-fx}

\subsection{Proofs of Tradeoff with Simple Equality}
\label{apx:proof-fx-eq}

We repeat the input representation again for the reader's convenience:
\begin{enumerate}
        \item $\Tm^0(\y, \z)_{i} = y_{i}$ for $i \in [m]$.
        \item $\Tm^0(\y, \z)_{i+m} = z_{i}$ for $i\in [m]$.
        \item $\Tm^0(\y, \z)_{i} = 0$ for all $i \in  [2m+1: n]$. 
    \end{enumerate}

We fix $m:= 2p-1$.

With this notation, we prove the lower bound using communication complexity. First, we prove a lower bound on $\EQ_m$ using communication complexity, and use that to find a lower bound on the required size of a Transformer that solves $\EQ_m$ using the above input representation.

\begin{lemma}[Communication Lower Bound]
    \label{lem:cc-eq}
    If Alice has $\y$ and Bob has $\z$, with $\y, \z \in \{0, 1\}^m$ satisfying the promise $\y\leq \z$, in the one-way communication model where only Alice can send messages to Bob, Alice needs to send at least $m$ bits to Bob in order for Bob to compute the function $g(\y, \z)$. 
\end{lemma}

\begin{proof}
    Let us assume that the communication complexity of this problem (equal to the number of bits sent by Alice to Bob), is at most $m-1$. Since the number of possible $\y$'s is $2^m$, by pigeon hole principle, there exists a single message sent by Alice to Bob which corresponds to at least $\frac{2^m}{2^{m-1}} = 2$ distinct values of $\y$.

    Let these two values of $\y$ be $\y^{(1)}$ and $\y^{(2)}$. Assuming without loss of generality that $\y^{(1)} < \y^{(2)}$, if Bob has the string $\y^{(1)}$, then the correct answer can be either $1$ or $0$ for the same transcript sent by Alice. This is a contradiction, which implies the communication complexity is $\geq m$.
\end{proof}

Using this communication lower bound, we can prove Lemma \ref{lem:lb-eq}, which is restated as follows.
    
\LBEQ*
\begin{proof}
    Consider a Transformer with precision $p$ and value dimension $d_v = 1$ having one head, that correctly computes $\EQ_m$ in $\Tm^0$ representation. We use this Transformer to create a communication protocol that uses $2p$ bits of communication for computing the equality function, $\EQ_m$. Subsequently, we will prove that any communication protocol solving $\EQ_m$ requires at least $m$ bits of communication, which would then imply $p\geq \mm$ for a Transformer to solve $\EQ_m$. 
    
    \paragraph{Transformer to communication.} Recalling the problem definition, we have $n+1$ tokens. We devise the communication problem such that Alice has the contiguous set of tokens $\{\Tm^0(\y, \z)_{-1},\Tm^0(\y, \z)_0, \Tm^0(\y, \z)_1, \ldots, \Tm^0(\y, \z)_{m}\}$ (note that these tokens only depend on $\y_1, \dots, \y_{m}$), and Bob has the rest (which depend on $\z$). For ease of notation, we define the contiguous set $S$ as the set of tokens with Alice.

    Since there is a Transformer computing $\EQ_m$ in the $\Tm^0$ representation, Alice and Bob can find the weight matrices of this Transformer separately. If there are multiple such Transformers, they just consider the lexicographically least setting of weights (a specific Transformer which they agree upon). The communication proceeds as follows:
    \begin{enumerate}
        \item Alice computes the values of 
        \begin{equation*}
            r_{n+1, j} = e^{ X_{n+1}W^Q (W^K)^T X_j },
        \end{equation*}
        for all $j\in S$ using left-associative multiplication.
        \item Bob computes the values of 
        \begin{equation*}
            r_{n+1, j} = e^{X_{n+1}W^Q (W^K)^T X_j},
        \end{equation*}
        for all $j\in [n+1]\backslash S$ using left-associative multiplication.
        \item Next, Alice sends Bob the values of 
        \begin{equation*}
            \begin{split}
                L_1 & = \sum_{j\in S}r_{n+1, j},\\
                L_2 & = \sum_{j\in S}r_{n+1, j}X_jW^V,
            \end{split}
        \end{equation*}
        using $p$ bits of precision each maintaining left-associative addition.
        \item Bob computes the value of 
        \begin{equation*}
            L = \frac{L_2 + \sum_{j\in [n+1]\backslash S}r_{n+1, j}V_j}{L_1 + \sum_{j\in [n+1]\backslash S}r_{n+1, j}},
        \end{equation*}
        approximated upto $p$ bits of precision.
        \item Then, Bob computes the value of $L$ applied to the final MLP layer, and outputs $1$ if the output of the first entry of the MLP is $1$, $0$ otherwise.

    \end{enumerate}

    The communication complexity of this protocol is $2p$.

    \paragraph{Wrapping up.} Therefore, we need $2p \geq m$ as there exists a communication protocol of cost $2p$, which implies $p \geq \mm$. 
\end{proof}

We state the upper bound theorem again.
\UBEQ*
\begin{proof} 
We give the construction of a Transformer that uses precision $p=\mm+1$ to compute $\EQ_m$ with input representation $\Tm^0$. 

\paragraph{Embedding.} We define the embedding function $\pe: [n+1]\times\{0, 1\}\rightarrow \R^{d}$, which takes as input the value of the token and the position index of the token, where $d_q=3$ is the embedding dimension,  such that the $i$-th token is given as:

\begin{table}[H]
\resizebox{\columnwidth}{!}{%
\begin{tabular}{c|c|c|c}
    Step & $i$ & $\Tm^0(\y, \z)_i = 0$ & $\Tm^0(\y, \z)_i = 1$ \\\hline
    1 & $i \in[6]$ & $\begin{bmatrix}
        1 && -6  && 0
    \end{bmatrix}$ & $\begin{bmatrix}
        1 && -6  && -2^{7-i}
    \end{bmatrix}$\\
    2 & $i\in [7:\mm]$ & $\begin{bmatrix}
        1 && -(i-2)  && 0
    \end{bmatrix}$ & $ \begin{bmatrix}
        1 && -(i-2)  && -1
    \end{bmatrix}$\\
    3 & $i\in [\mm+1: m]$ & $\begin{bmatrix}
                1 && (2\mm-i-1)  && 0
            \end{bmatrix}$  & $\begin{bmatrix}
                1 && -N  && 0
            \end{bmatrix}$\\
    4 & $i \in [m+1: m+6]$ & $\begin{bmatrix}
        1 && -6  && 0
    \end{bmatrix}$ & $\begin{bmatrix}
        1 && -6  && 2^{7+m-i}
    \end{bmatrix}$\\
    5 & $i\in [m+7:m+\mm + 1]$ & $\begin{bmatrix}
        1 && -(i-m-4)  && 0
    \end{bmatrix}$ & $ \begin{bmatrix}
        1 && -(i-m-4)  && 1
    \end{bmatrix}$\\
    6 & $i \in [m+\mm+1:2m]$ & $\begin{bmatrix}
                1 && -N  && 0
            \end{bmatrix}$ & $\begin{bmatrix}
                1 && (m+2\mm-i+1)  && 0
            \end{bmatrix}$ \\
    7 & $i = 2m+1$  & $\begin{bmatrix} 
        1 && -1 && 2-2^{-(\mm-2)}
    \end{bmatrix}$ & NA\\
    8 & $i \in [2m+2: n-1]$ & $\begin{bmatrix}
        1 && -N && 0
    \end{bmatrix}$ & NA
\end{tabular}%
}
\label{tab:pe-eq}
\caption{Embedding function.}
\end{table}

    We choose $N$ as a large enough number, whose exponentiation after negation will be rounded off to zero in bounded precision computations. These notions are commonly used in programming languages, and Python has the value \textsf{float(`inf')} (i.e., $2^N$ is \textsf{float(`inf')}).

    \paragraph{Construction of the Weights.} We define the query weights $W^Q \in \R^{d\times d_q}$, for query dimension  $d_q = 1$, as,
    \begin{equation*}
        W^Q = \ln 2\begin{bmatrix}
             1  \\
             0  \\
             0   
        \end{bmatrix},
    \end{equation*}
    the key weights $W^K\in \R^{d\times d_q}$, as,
    \begin{equation*}
        W^K = \begin{bmatrix}
              0 \\
              1 \\
              0 
        \end{bmatrix},
    \end{equation*}
    and the value weights $W^V \in \R^{d\times d_v}$, for value dimension $d_v = 1$, as,
    \begin{equation*}
        W^V = \begin{bmatrix}
            
            0 \\
            0 \\
            2^{\mm-1}
        \end{bmatrix}.
    \end{equation*}

    Now, we analyze the value of the self-attention output for YES and NO instances by separately looking at the numerator term and the denominator term.
    
    \paragraph{Numerator.} From the above construction of the embeddings and the weight matrices, we first use left-associative multiplication to first multiply $X_{n+1}$ with $W^Q$, followed by $(W^K)^T$, followed by $X_i^T$, after which, we perform exponentiation and multiply it to $X_i$, then $W^V$.

    With these values, the summation in the numerator $\sum_i e^{\langle Q_{n+1}, K_i\rangle}X_iW^V$ (following the low-bit computations for self-attention conventions defined in \cref{apx:comp}) after Step 1, 2, becomes $ 2^{\mm-1}(-y_1) + \ldots + 2^0 (-y_{\mm}) = -\y_{1:\mm}$ which is a number within the permissible range. In Step 3, the values from $X_iW^V$ are zero, which gives $\sum_{i\in [\mm+1, m]} e^{\langle Q_{n+1}, K_i\rangle}X_iW^V = 0$. 

    After this, values from $\z$ are added. Steps 4, 5 gives the sum of $\sum_i e^{\langle Q_{n+1}, K_i\rangle}V_i$ as $2^{\mm-1}z_1 + \ldots + 2^0 z_{\mm} = \z_{1:\mm}$, and Step 6 adds $0$. Finally, Step 7 gives the value $2^{\mm}-1$. 

    Therefore, if $\y_{1:\mm} = \z_{1:\mm}$, then the value of the numerator is $2^{\mm}-1$, otherwise, it is a strictly larger value, which is not representable in this precision format, and thus becomes infinity (\textsf{float(`inf')}).

    Here, the fixed-point scaling constant is 1.

    \paragraph{Denominator.} The denominator is given by $\sum_{i=1}^{n+1}e^{ \langle Q_{n+1}, K_i\rangle }$. After Steps 1, 2, this becomes 
    \begin{equation}
    \label{eqn:0}
        \begin{split}
            \frac{1}{2^{\mm-1}} +  \frac{7}{2^6} +  \sum_{j=6}^{\mm-1} \frac{1}{2^{j}} = \frac{9}{2^6}.
        \end{split}
    \end{equation}

    Step 3 yields the number $\bar y_{\mm+1}\ldots \bar y_m = \bar \y_{\mm+1:m}$, and when added to the value from Steps 1, 2, results in the value of Equation \ref{eqn:0} being ignored due to rounding off. 

    Again, Steps 4, 5 give the same value as Equation \ref{eqn:0}, which is again rounded off, and finally Step 6 gives $ z_{\mm+1}\ldots  z_m = \z_{\mm+1:m}$. Finally, Step 7 gives $2^{-1}$, and the rest are ignored.
    
    If $\y_{\mm+1:m} = \z_{\mm+1:m}$, the value of the denominator is $\bar \y_{\mm+1:m} + \z_{\mm+1:m} + 2^{-1} = 2^{-1}(2^{\mm}-1)$ (since $m$ is odd and $m-\mm = \mm-1$). 
    
    Otherwise, if $\y_{\mm+1:m} < \z_{\mm+1:m}$, we show that the value of the denominator will be strictly greater than $2^{\mm-1}$. Since the denominator is $\bar \y_{\mm+1:m} + \z_{\mm+1:m}$, this can be written as $\bar \y_{\mm+1:m} + \y_{\mm+1:m} - \y_{\mm+1:m} + \z_{\mm+1:m} \geq \bar \y_{\mm+1:m} + \y_{\mm+1:m}+1 = 2^{\mm-1}$.
    
    Therefore, the denominator becomes $2^{-1}(2^{\mm}-1)$ when $\y_{\mm+1:m} = \z_{\mm+1:m}$, otherwise a larger value.

    Here, the fixed-point scaling constant is $2^{-1}$.

    \paragraph{Wrapping up.} For the YES instance, the value of the numerator is $2^{\mm}-1$, and the denominator is $2^{-1}(2^{\mm-1}-1)$, which gives the output of the self-attention as $1/2$. 
    
    For the NO instance, either the numerator will be infinity, or the numerator will be $(2^{\mm}-1)$ and the denominator will be larger than $2^{-1}(2^{\mm}-1)$. We now show that this change can be detected in $\mm+1$ bits of fixed-point precision.

    Let us assume the denominator is $x$ larger than the expected value $2^{-1}(2^{\mm}-1)$. Therefore, the output of self-attention will be
    \begin{equation*}
        \begin{split}
            \frac{(2^{\mm}-1)}{2^{-1}(2^{\mm}-1) + x} & = 2\left(1 + \frac{2x}{2^{\mm}-1}\right)^{-1}\\
            & \approx 2\left(1 + \frac{2x}{2^{\mm}}(1 + \frac{1}{2^{\mm}} + (\frac{1}{2^{\mm}})^2 + \ldots )\right)^{-1}\\
            & \approx 2\left( 1 + \frac{x}{2^{\mm-1}} \right)^{-1} \\
            & \approx 2\left( 1 - \frac{x}{2^{\mm-1}} + (\frac{x}{2^{\mm-1}})^{2} - \ldots \right),
        \end{split}
    \end{equation*}
    which is a number that can be distinguised from $2$ using $\mm+1$ bits of precision.

    \paragraph{MLP layer.} We construct a simple MLP, $\R \rightarrow \R$, that returns $1$ when the given input is $2$, otherwise 0. This has only one hidden layer, which has only two nodes. Using weights $$w^{1}_{1, 1} = -1, b^1_1 = 1/2,$$ and, $$w^{1}_{1, 2}=1, b^1_2 = -1/2,$$ we can check, using $\RELU$, if the output of the soft-max is exactly $2$. For the next layer, with weights $$w^{2}_{1,1} = -2^{\mm}, w^{2}_{2, 1} = -2^{\mm}, b^2_1 = 1,$$ the output after $\RELU$ will be exactly $1$ for YES instances and $0$ for NO instances.\footnote{This works even when the value is \textsf{float(`inf')}.}
\end{proof}

\subsection{Tight Tradeoff for Fixed-Point Precision}
\label{apx:rem-proofs}

In this section, we achieve a tight tradeoff with precision. For the reader's convenience, we repeat the input representation.

The input representation $\Tm^1 : \{0, 1\}^{2m}\rightarrow \{0, 1\}^n$ is defined as follows:
        \begin{enumerate}
        \item $\Tm^1(\y, \z)_{-1} = 0$.
        \item $\Tm^1(\y, \z)_0 = y_1$.
        \item $\Tm^1(\y, \z)_{i} = y_{i}$ for $i \in [m]$.
        \item $\Tm^1(\y, \z)_{i} = z_1$ for $i \in \{m+1, m+2\}$.
        \item $\Tm^1(\y, \z)_{m+3} = 0$.
        \item $\Tm^1(\y, \z)_{i} = z_{i-m-2}$ for $i\in [m+4: 2m+2]$.
        \item $\Tm^1(\y, \z)_{i} = 0$ for all $i \in  [2m+3: n-2]$. 
    \end{enumerate}

    The input tokens can be seen as:
    \begin{equation*}
        0,y_1,y_1,y_2,\ldots,y_m,z_1,z_1,0,z_2,z_3,\ldots,z_m,0,0,\ldots,0.
    \end{equation*}

    We fix $m := 2p-1$, and prove the lower and upper bound for \cref{thm:tradeoff-main}.

\LBEQFX*

\begin{proof}
    This proof is exactly same as that of Lemma \ref{lem:lb-eq}, where Alice's set $S$ is now the set of tokens $\{-1, 0, \ldots, m\}$, and Bob has the rest.
\end{proof}

We state the upper bound theorem again.

\UBEQFX*

\begin{proof} 
We give the construction of a Transformer that uses precision $p=\mm$ to compute $T_m$ at an input $\Tm^3(\y, \z)$. 

\paragraph{Embedding.} We define the embedding function $\pe: [n]\times\{0, 1\}\rightarrow \R^{d}$, where $d=3$ is the embedding dimension,  such that the $i$-th token is given as:

\begin{table}[H]
\resizebox{\columnwidth}{!}{%
\begin{tabular}{c|c|c|c}
    Step & $i$ & $\Tm^1(\y, \z)_i = 0$ & $\Tm^1(\y, \z)_i = 1$ \\\hline
    1 & $i = -1$ & $\begin{bmatrix}
        1 && -(\mm-1) && -(2^{\mm-1}-1)
    \end{bmatrix}$ & NA\\
    2 & $i \in \{0, 1\}$ & $\begin{bmatrix}
        1 && -6  && 0
    \end{bmatrix}$ & $\begin{bmatrix}
        1 && -6  && 2^6
    \end{bmatrix}$\\
    3 & $i =2$ & $\begin{bmatrix}
        1 && -6  && 0
    \end{bmatrix}$ & $\begin{bmatrix}
        1 && -6  && 2^6
    \end{bmatrix}$\\
    4 & $i =3$ & $\begin{bmatrix}
        1 && -6  && 0
    \end{bmatrix}$ & $\begin{bmatrix}
        1 && -6  && 2^5
    \end{bmatrix}$\\
    5 & $i =4$ & $\begin{bmatrix}
        1 && -6  && 0
    \end{bmatrix}$ & $\begin{bmatrix}
        1 && -6  && 2^4
    \end{bmatrix}$\\
    6 & $i =5$ & $\begin{bmatrix}
        1 && -6  && 0
    \end{bmatrix}$ & $\begin{bmatrix}
        1 && -6  && 2^3
    \end{bmatrix}$\\
    7 & $i =6$ & $\begin{bmatrix}
        1 && -6  && 0
    \end{bmatrix}$ & $\begin{bmatrix}
        1 && -6  && 2^2
    \end{bmatrix}$\\
    8 & $i\in [7:\mm]$ & $\begin{bmatrix}
        1 && -(i-2)  && 0
    \end{bmatrix}$ & $ \begin{bmatrix}
        1 && -(i-2)  && 1
    \end{bmatrix}$\\
    9 & $i\in [\mm+1: m]$ & $\begin{bmatrix}
                1 && (2\mm-i-1)  && 0
            \end{bmatrix}$  & $\begin{bmatrix}
                1 && -N  && 0
            \end{bmatrix}$\\
    10 & $i \in \{m+1, m +2\}$ & $\begin{bmatrix}
        1 && -6  && 0
    \end{bmatrix}$ & $\begin{bmatrix}
        1 && -6  && -2^6
    \end{bmatrix}$\\
    11 & $i = m+3$ & $\begin{bmatrix}
        1 && -(\mm-1) && (2^{\mm-1}-1)
    \end{bmatrix}$ & NA \\
    12 & $i =m+4$ & $\begin{bmatrix}
        1 && -6  && 0
    \end{bmatrix}$ & $\begin{bmatrix}
        1 && -6  && -2^6
    \end{bmatrix}$\\
    13 & $i =m+5$ & $\begin{bmatrix}
        1 && -6  && 0
    \end{bmatrix}$ & $\begin{bmatrix}
        1 && -6  && -2^5
    \end{bmatrix}$\\
    14 & $i =m+6$ & $\begin{bmatrix}
        1 && -6  && 0
    \end{bmatrix}$ & $\begin{bmatrix}
        1 && -6  && -2^4
    \end{bmatrix}$\\
    15 & $i =m+7$ & $\begin{bmatrix}
        1 && -6  && 0
    \end{bmatrix}$ & $\begin{bmatrix}
        1 && -6  && -2^3
    \end{bmatrix}$\\
    16 & $i =m+8$ & $\begin{bmatrix}
        1 && -6  && 0
    \end{bmatrix}$ & $\begin{bmatrix}
        1 && -6  && -2^2
    \end{bmatrix}$\\
    17 & $i\in [m+9:m+\mm + 2]$ & $\begin{bmatrix}
        1 && -(i-m-4)  && 0
    \end{bmatrix}$ & $ \begin{bmatrix}
        1 && -(i-m-4)  && -1
    \end{bmatrix}$\\
    18 & $i \in [m+\mm+3:2m+2]$ & $\begin{bmatrix}
                1 && -N  && 0
            \end{bmatrix}$ & $\begin{bmatrix}
                1 && (m+2\mm-i+1)  && 0
            \end{bmatrix}$ \\
    19 & $i = 2m+3$  & $\begin{bmatrix} 
        1 && -(\mm-1) && 1
    \end{bmatrix}$ & NA\\
    20 & $i \in [2m+4: n-1]$ & $\begin{bmatrix}
        1 && -N && 0
    \end{bmatrix}$ & NA
\end{tabular}%
}
\label{tab:pe-1}
\caption{Embedding function}
\end{table}

    Again, $N$ as a large enough number as was chosen in the proof of Lemma \ref{lem:ub-fixed-pt-eq}.

    \paragraph{Construction of the Weights.} We define the query weights $W^Q \in \R^{d\times d_q}$, for query dimension  $d_q = 1$, as,
    \begin{equation*}
        W^Q = \ln 2\begin{bmatrix}
             1  \\
             0  \\
             0   
        \end{bmatrix},
    \end{equation*}
    the key weights $W^K\in \R^{d\times d_q}$, as,
    \begin{equation*}
        W^K = \begin{bmatrix}
              0 \\
              1 \\
              0 
        \end{bmatrix},
    \end{equation*}
    and the value weights $W^V \in \R^{d\times d_v}$, for value dimension $d_v = 1$, as,
    \begin{equation*}
        W^V = \begin{bmatrix}
            
            0 \\
            0 \\
            2^{\mm-2}
        \end{bmatrix}.
    \end{equation*}

    Now, we analyze the value of the self-attention output for YES and NO instances by separately looking at the numerator term and the denominator term.
    
    \paragraph{Numerator.} Again, we compute the self-attention in bounded-precision (\cref{apx:comp}). From the above construction of the embeddings and the weight matrices, 
    the summation in the numerator $\sum_i e^{\langle Q_{n+1}, K_i\rangle}V_i$ after Step 1, 2, becomes $ - \underbrace{1\ldots 1}_{\mm-1}+y_1 \underbrace{0\ldots 0}_{\mm-1},$ which is a number within the permissible range. After Steps  3-8, the number $y_{2}\ldots y_{\mm}$ is added to it, which gives the answer as $$y_1y_2\ldots y_{\mm} - \underbrace{1\ldots 1}_{\mm-1}.$$ This number is within the permissible range of fixed-point precision. 

    Step 9 does not add anything, and after this, values from $\z$ are added. First, Step 10 subtracts $2^{\mm-2}\times z_1$ twice to this. Here, if $y_1 = z_1$, then the number stays in the permissible range, otherwise if $y_1 = 0$ and $z_1 = 1$, this becomes $-\textsf{float(`inf')}$ ($y_1 = 1, z_2=0$ is not possible according to the promise). Moving forward, we can assume $y_1 = z_1$, as otherwise, we will see that the MLP layer in the end can check if this does not happen.

    In Steps 11 to 17, the value of the summation will be 
    \begin{equation*}
        \underbrace{1\ldots 1}_{\mm-1} - z_2\ldots z_{\mm},
    \end{equation*}
    Step 18 givs $0$ and Step 19 gives $1$.

    Therefore, the total value of the numerator is $1$ if $\y_{1:\mm} = \z_{1:\mm}$, otherwise it will be negative. The scaling factor is $1$, as before.

    \paragraph{Denominator.} The denominator is given by $\sum_{i=1}^{n+1}e^{ X_{n+1}W^Q (W^K)^TX_i^T }$. Step 1 will add $2^{-(\mm-1)}$, and after Steps 2-8 this becomes 
    \begin{equation}
    \label{eqn:1}
        \begin{split}
            \frac{1}{2^{\mm-1}} +  \frac{7}{2^6} +  \sum_{j=6}^{\mm-1} \frac{1}{2^{j}} = \frac{9}{2^6}.
        \end{split}
    \end{equation}

    Step 9 is equal to the number $\bar y_{\mm+1}\ldots \bar y_m$, and when added to the result from Steps 1-8 results in the value of Equation \ref{eqn:1} being ignored due to rounding off. 

    Again, Steps 10 to 17 gives the same value as Equation \ref{eqn:1}, which is again rounded off. Next, Step 18 is the value $ z_{\mm+1}\ldots \bar z_m$. 
    
    If $\y = \z$, the value of the denominator is $(2^{\mm-1}-1) = \underbrace{1\ldots 1}_{\mm-1}$. Otherwise, there will be some minimum $i_0$ such that 
    $y_i = z_i$ for $i < i_0$ and $y_{i_0} = 0, z_{i_0} = 1$. In this case, from Steps 9 and 18, the value will be:
    \begin{equation*}
        \sum_{i < i_0} \bar y_{\mm+i}2^{\mm-1-i} + 2^{\mm-1-i_0} + \ldots + \sum_{i < i_0} z_{\mm+i}2^{\mm-1-i} + 2^{\mm-1-i_0} + \ldots,
    \end{equation*}
    where the ignored terms are all positive. Since $y_i = z_i$ for $i < i_0$, this number becomes larger than $(2^{\mm-1}-1)$, which is \textsf{float(`inf')} in the given precision format. Therefore, the denominator becomes $2^{\mm-1}-1$ when $\y_{\mm+1:m} = \z_{\mm+1:m}$, otherwise \textsf{float(`inf')}. This again has a fixed-point scaling equal to $1$.

    \paragraph{Wrapping up.} For the YES instance, the value of the numerator is $1$ (if it is 0 or negative we will instantly reject), and the denominator is $2^{\mm-1}-1$, which gives the output of the self-attention as $2^{-(\mm-1)}$. 
    
    For the NO instance, either the numerator will be non-positive, or the numerator will be $1$ and the denominator will be \textsf{float(`inf')}, in which case, the softmax will be either negative or $0$.

     This can be separated from $\frac{1}{2^{\mm-1}}$ with $\mm$ bits of precision.

    \paragraph{MLP layer.} We construct a simple MLP, $\R \rightarrow \R$, that returns $1$ when the given input is $2^{-(\mm-1)}$, otherwise 0. This has only one hidden layer, which has only two nodes. Using weights $$w^{1}_{1, 1} = -1, b^1_1 = 1/2^{\mm-1},$$ and, $$w^{1}_{1, 2}=1, b^1_2 = -1/2^{\mm-1},$$ we can check, using $\RELU$, if the output of the soft-max is exactly $1/2$. For the next layer, with weights $$w^{2}_{1,1} = -2^{\mm-1}, w^{2}_{2, 1} = -2^{\mm-1}, b^2_1 = 1,$$ the output after $\RELU$ will be exactly $1$ for YES instances and $0$ for NO instances.
\end{proof}

\section{Tradeoff with Floating-Point Precision}
\label{apx:fp}

In this section, we prove the remaining results.

We first show a tight tradeoff for linear self-attention Transformers (see Definition \ref{def:linear-att}), which are a simplified version of the usual self-attention Transformer. They have been widely used in theoretical analyses of Transformers \cite{keles2023computational,ustaomeroglu2025theoretical,ahn2023linear,schlag2021linear}. This simplified structure, even though less expressive, sometimes tolerable loss in accuracy. Yet, they are quite efficient and have a linear running time, and have been widely used to devise efficient approximation algorithms for self-attention Transformers \cite{choromanski2020rethinking,katharopoulos2020Transformers,koohpayegani2024sima,kacham2023polysketchformer,ramapuram2024theory,alman2023fast,wang2020linformer}. Recently, Qwen3 \cite{yang2025qwen3} released an architecture which a series of large language models where linear-attention is used as Gated DeltaNets.

The following are the main tradeoff results for floating-point precision quantizations.

\begin{theorem}[Tradeoff for Linear Transformers]
    \label{thm:tradeoff-linear}
    For every integer $p> 0$ and every integer $n > 2p$, choosing $m = p$, there exists an input representation $\Tm^2$ such that:
    \begin{enumerate}
        \item No linear Transformer with $n$ tokens having value dimension $d_v=1$ and one head, can compute $\EQ_m$ with the input representation $\Tm^2$, using $ p-1$ bits of precision, 
        \item There exists a linear Transformer with $n$ tokens having value dimension $d_v = 1$ and one head, that can compute ${\EQ_m}$ with the input representation $\Tm^2$ using $p$ bits of precision in the floating-point precision format. 
    \end{enumerate}
\end{theorem}

We can show a similar result of expressivity tradeoff with precision for soft-max Transformers as well, but with a logarithm gap.

\begin{theorem}[Tradeoff for Softmax Transformers]
    \label{thm:tradeoff-fp}
    For every integer $p> 0$ and every integer $n > 4p$, choosing $m = p$, there exists an input representation $\Tm^3$ such that:
    \begin{enumerate}
        \item No linear Transformer with $n$ tokens having value dimension $d_v=1$ and one head, can compute $\EQ_m$ with the input representation $\Tm^3$, using $ p-1$ bits of precision, 
        \item There exists a linear Transformer with $n$ tokens having value dimension $d_v = 1$ and one head, that can compute ${\EQ_m}$ with the input representation $\Tm^3$ using $p$ bits of precision in the floating-point precision format. 
    \end{enumerate}
\end{theorem}

\subsection{Tradeoff with Floating-Point Precision: Linear Self-Attention}

In this section, we prove an exact tradeoff for linear self-attention. Again, the function will depend on strings $\mathbf{y}, \mathbf{z}\in \{0, 1\}^m$ with the promise that:
\begin{itemize}
    \item The binary values satisfy $\y \leq \z$.
    \item $\y_{2:e+1} > 0$.
    \item $y_{m-1}y_{m}\neq 10$.
\end{itemize}

Since floating-point precision format is associated with a mantissa and an exponent, we fix a tuple $(t, e)$, where $t, e$ are any positive integers such that $t+e = m$ and $e > 1$, on which we will show that the upper-bound follows. Based on this, we define the input representation $\Tm^2: \{0, 1\}^{2m} \rightarrow\{0, 1\}\times \{0, 1\}^{e+1} \times \{0, 1\}$ as:
\begin{itemize}
    \item $\Tm^2(\y, \z)_{0} = (y_1, (y_2, \ldots, y_{e+1}), y_{e+3})$.
    \item $\Tm^2(\y, \z)_i = (y_1, (y_2, \ldots, y_{e+1}), 0)$ for $i\in \{3, 5, \ldots, 2t-7\}$.
    \item $\Tm^2(\y, \z)_i = (y_1, (y_2, \ldots, y_{e+1}), y_{e+3+\frac{i}{2}})$ for $i \in \{ 2, 4, \ldots, 2t-8\}$.
    \item $\Tm^2(\y, \z)_{2t-6} = (y_1, (y_2, \ldots, y_{e+1}), \bar y_{e+t})$.
    \item $\Tm^2(\y, \z)_{2t-5} = (y_1, (y_2, \ldots, y_{e+1}), y_{e+2})$.
    \item $\Tm^2(\y, \z)_i = (0, 0, 0)$ for $i \in [2t-4, 2t-1]$.
    \item $\Tm^2(\y, \z)_{2t} = (z_1, (z_2, \ldots, z_{e+1}), z_{e+t})$.
    \item $\Tm^2(\y, \z)_i = (z_1, (z_2, \ldots, z_{e+1}), z_{e+\frac{i-2t+3}{2}})$ for $i\in \{2t+1, 2t+3, \ldots, 4t-5\}$.
    \item $\Tm^2(\y, \z)_{i} = (z_1, (z_2, \ldots, z_{e+1}), 0)$ for $i\in \{2t+2, 2t+4, \ldots, 4t-6\}$.
    \item $\Tm^2(\y, \z)_{i} = 0$ for $i\in [4t-6: n]$
\end{itemize}
Here, $\Tm^2(\y, \z)^1_i$ will refer to the first element of the tuple, $\Tm^2(\y, \z)^2_i$ will refer to the second element (which itself is a $e-1$ length tuple), and $\Tm^2(\y, \z)^3_i$ as the third element.

Even though this construction might seem complicated, this has been inspired by the upper bound of the floating-point precision format, where the first element of the tuple $\Tm^2(\y, \z)^1_i$ will be used in sign bit of the mantissa, $\Tm^2(\y, \z)^2_i$ will be used in the (signed) exponent, and $\Tm^2(\y, \z)^3_i$ will be used to construct bits of the mantissa.

We fix $m := p$ and using the following results for lower bounds and upper bounds, prove \cref{thm:tradeoff-linear}.

\begin{lemma}[Communication Bound]
\label{lem:cc-F}
    If Alice has $\y$ and Bob has $\z$ in the one-way communication where only Alice can send messages to Bob, Alice needs to send at least $m$ bits to Bob in order for Bob to compute the function $g(\y, \z)$.
\end{lemma}

\begin{proof}
    We again prove this lower bound using the fooling set method. Define the set of vectors:
    \begin{equation*}
        F = \{(x, x)\,|\, x\in \{0, 1\}^m, ~x_{m-1}x_m \neq10,\, x_2\ldots x_e\neq 0\ldots 0 \}.
    \end{equation*}
     Consider two distinct tuples $(x, x), (x', x')\in F$. Since $x\neq x'$, we either have $x_{1:\mm} < x'_{1:\mm}$, or $x'_{1:\mm} < x_{1:\mm}$, while at the same time $x_{m-1}x_{m} \neq 10$, $x'_{m-1}x'_{m} \neq 10$ and both $x_{2:e}$ and $x'_{2:e}$ are nonzero. This implies that $F$ is a fooling set.

     The size of this fooling set is given by $(2^2-1)(2^{e-1}-1)\times 2^{m-e-1} = 3\times 2^{m-2}\times (1-2^{-e})$. Therefore, the communication complexity lower bound is given by $\lceil \log (3.2^{m-2}(1-2^{-e}))\rceil = m$ for $e > 1$.
\end{proof}

\begin{lemma}[Transformer Lower Bound]
\label{lem:lb-linear}
    Given the input representation $\Tm^2$, any linear Transformer with $d_v=1$ and $H=1$, that computes $\EQ_m$ requires precision $ \geq m$ bits.
\end{lemma}

\begin{proof}
    This lower bound is again exactly same as the proof of Lemma \ref{lem:lb-eq}, where we first assume that there is a linear Transformer that computes this function, and devise a communication protocol from it. The only differences are:
    \begin{itemize}
        \item The set $S$ is chosen as $\{0, 1, \ldots, 2t-1\}$.
        \item Instead of sending 2 messages, Alice sends one message $\sum_{i\in S} X_{n+1}W^Q (W^K)^T X_{i}^T X_i W^V$ (see definition of linear attention in Definition \ref{def:linear-att}).
    \end{itemize}

    Using this, we can show the Transformer requires precision at least $m$.
\end{proof}

Next, we give a construction that shows that this lower bound is actually tight.

\begin{lemma}[Upper Bound]
\label{lem:ub-linear}
There exists a linear Transformer with $d_v=1$, $H=1$, that uses precision $(t, e)$ in the floating-point precision format, where $t+e = m$, to compute $\EQ_m$ in the input representation $\Tm^2$.
\end{lemma}

\begin{proof}
    The idea as before, is to again generate binary numbers from $\mathbf{y}$ and $\mathbf{z}$ such that they cancel out if they are equal, negative otherwise.

    \paragraph{Overview.} The construction proceeds in two steps, first we create a number corresponding to $\mathbf{y}$, which is $(-1)^{y_1}\times 1\cdot(y_{e+2}\ldots y_{m-1}\bar y_{m})\times 2^{\y_{2:e+1} - (2^{e-1}-1)}$ (note that the last bit is flipped). Such an expression is used to utilize all the bits used in the precision format. Note that even though $\y_{2:e+1}$ requires $e$ bits, we subtract $q$ to it to bring it between $-(2^{e-1}-1)$ to $+(2^{e-1}-1)$. This, along with the sign bit, is a valid representation of the exponent in floating-point precision.

    However, creating this number $(-1)^{y_1}\times 1\cdot(y_{e+2}\ldots y_{m-1}\bar y_m)\times 2^{\y_{2:e+1} - (2^{e-1}-1)}$ is a challenging task and the one needs to be careful so as not to lose any bits to rounding off. In order to add the bit $y_i$ to the $i$-th place after decimal ($i\in [e+2: m]$), we add $(-1)^{y_1}\times 1\cdot0\ldots0y_i\times 2^{\y_{2:e+1}-q}$, and after each addition, we remove the extra $1\cdot0$ in the addition by subtracting $-(-1)^{y_1}\times 1\cdot0\times 2^{\y_{2:e+1}-q}$.

    During this process, in order to prevent the entire number becoming zero (which might happen when we subtract in the form, say $1\cdot0001-1\cdot0$), we force the first digit after $1$ to contain a $1$, which we remove and make it $y_{e+2}$ later.

    Finally, we show that if the numbers are equal, the summation $\sum_i e^{\langle Q_{n+1}, K_i\rangle}V_i$ will be $0$ when $\mathbf{y} = \mathbf{z}$, and any other non-zero number otherwise.

    \paragraph{Embedding.} We define the embedding as follows. Also we define $q:= 2^{e-1}-1$.

    \begin{table}[!h]
        \resizebox{\columnwidth}{!}{%
        \begin{tabular}{c|c|c|c}
             Step & $i$ & $\Tm^2(\y, \z)_i^3 = 0$ & $\Tm^2(\y, \z)_i^3 = 1$\\ \hline
             1 & $i=0$ & $\begin{bmatrix}
                 1 && {\y_{2:e+1}-q} && 0
             \end{bmatrix}$ & $\begin{bmatrix}
                 1 && {\y_{2:e+1}-q} && (-1)^{y_1}\times (1+ 2^{-1}+ 2^{-2})
             \end{bmatrix}$\\
             2 & $i\in \{1, 3, \ldots, 2t-7\} $ & $\begin{bmatrix}
                 1 && {\y_{2:e+1}-q} && -(-1)^{y_1}
             \end{bmatrix}$ & NA\\
             3 & $i\in \{2, \ldots, 2t-6\} $ & $\begin{bmatrix}
                 1 && {\y_{2:e+1}-q} && 0
             \end{bmatrix}$ & $\begin{bmatrix}
                 1 && {\y_{2:e+1}-q} && (-1)^{y_1}\times (1 + 2^{-(\frac{i+4}{2})})
             \end{bmatrix}$\\
             4 & $i=2t-5$ & $\begin{bmatrix}
                 1 && \y_{2:e+1} && -(-1)^{y_1}\times 2^{-1}
             \end{bmatrix}$ & $\begin{bmatrix}
                 1 && \y_{2:e+1} && 0
             \end{bmatrix}$ \\
             5 & $i\in [2t-4, 2t-1]$ & $\begin{bmatrix}
                 1 && 0 && 0
             \end{bmatrix}$ & NA\\
             6 & $i=2t$  & $\begin{bmatrix}
                 1 && {\z_{2:e+1}-q} && (-1)^{z_1} 
             \end{bmatrix}$ & $\begin{bmatrix}
                 1 && {\z_{2:e+1}-q} && (-1)^{z_1}(1 + 2^{-(t-1)})
             \end{bmatrix}$\\
             7 & $i = \{2t+1, 2t+3, \ldots, 4t-5\}$ & $\begin{bmatrix}
                 1 && {\z_{2:e+1}-q} && -(-1)^{z_1}
             \end{bmatrix}$  & $\begin{bmatrix}
                 1 && {\z_{2:e+1}-q} && -(-1)^{z_1}(1 + 2^{-(\frac{i-2t+1}{2})})
             \end{bmatrix}$  \\
             8 & $i = \{2t+2, 2t+4, \ldots, 4t-6\}$ & $\begin{bmatrix}
                 1 && {\z_{2:e+1}-q} && (-1)^{z_1}
             \end{bmatrix}$  & NA\\
        \end{tabular}%
        }
        \caption{Embedding function}
        \label{tab:pe-2}
    \end{table}

\paragraph{Construction of the Weights.} We define the query weights $W^Q \in \R^{d\times d_q}$, for query dimension  $d_q = 1$, as,
    \begin{equation*}
        W^Q = \ln 2\begin{bmatrix}
             1  \\
             0  \\
             0   
        \end{bmatrix},
    \end{equation*}
    the key weights $W^K\in \R^{d\times d_q}$, as,
    \begin{equation*}
        W^K = \begin{bmatrix}
              0 \\
              1 \\
              0 
        \end{bmatrix},
    \end{equation*}
    and the value weights $W^V \in \R^{d\times d_v}$, for value dimension $d_v = 1$, as,
    \begin{equation*}
        W^V = \begin{bmatrix}
            
            0 \\
            0 \\
            1
        \end{bmatrix}.
    \end{equation*}
    Therefore, for an input token $\begin{bmatrix}
        a_1 & a_2 & a_3 
    \end{bmatrix}$, the corresponding row of the query, key and value matrices are
    \begin{equation*}
        \begin{split}
            & \begin{bmatrix}
        a_1 & a_2 & a_3  
    \end{bmatrix} W^Q =  \ln 2 \begin{bmatrix}
        a_1 
    \end{bmatrix},\\
    & \begin{bmatrix}
        a_1 & a_2 & a_3  
    \end{bmatrix} W^K = \begin{bmatrix}
        a_2
    \end{bmatrix},\\
    & \begin{bmatrix}
        a_1 & a_2 & a_3  
    \end{bmatrix} W^V = \begin{bmatrix}
        a_3
    \end{bmatrix}.
        \end{split}
    \end{equation*}

\paragraph{Wrap up.} We first show that steps 1-4 correctly computes the value $(-1)^{y_1}\times 1\cdot y_{e+2}\ldots y_{e+t}\times 2^{\y_{2:e+1}-q}$. 

Step 1 creates the value $(-1)^{y_1}\times 1\cdot1y_{e+3}\times 2^{\y_{2:e+1}-q}$, after which Step 2 subtracts the $1$ before decimal and Step 3 adds the $1\cdot\underbrace{0\ldots 0}_{i-(e+2)}y_{i}$ alternately in the left-associative sum. The $1$ after decimal prevents the round-off from going to zero and keeps it at most $(-1)^{y_1}\times 0\cdot1\ldots\times 2^{\y_{2:e+1}-q} = (-1)^{y_1}\times 1\cdot \ldots\times 2^{\y_{2:e+1}-q-1}$. Since $(y_2\ldots y_{e+1}) > 0$, this does not get rounded off to zero.

Finally, after Steps 2, 3 finish, we obtain $(-1)^{y_1}\times 1\cdot1y_{e+2}\ldots y_{e+t-1}\bar y_{e+t}\times 2^{\y_{2:e+1}-q}$, and Step 4 gives us $(-1)^{y_1}\times 1\cdot y_{e+1}\ldots  y_{e+t-1}\bar y_{e+t}\times 2^{\y_{2:e+1}-q}$.

Step 5 adds 0, and has been included only for notational convenience.

It can be verified that the number created by Steps 6-8 is positive. The only possibility of $y_1 \neq z_1$ is when $y_1 = 0, z_1 = 1$, in which case the entire sum will be a strictly positive number and never zero. Therefore, we can assume that $y_1 = z_1$.

We will show that Steps 7-8 create the number $(-1)^{z_1}\times 1\cdot\bar z_{e+1}\ldots \bar z_{e+t-1}\times 2^{\z_{2:e+1}-q}$, and Step 6 is essentially $(-1)^{z_1}\times 1\cdot\underbrace{0\ldots 0}_{t-2}\bar z_{e+t}\times 2^{\z_{2:e+1}-q}$. When $y_1 = z_1$, if $y_{2:e+1} < z_{2:e+1}$, this number is never zero. Therefore, if $y_1\neq z_1$ or $y_{2:e+1} \neq z_{2:e+1}$, $\mathbf{y}$ and $\mathbf{z}$ not being equal can be easily checked. Moving forward, we now assume $y_{1:e+1} = z_{1:e+1}$.

We now consider the following cases:\\

\textbf{Case 1: $\mathbf{y} = \mathbf{z}$:}\\
In this case, both $y_{e+t}$ and $z_{e+t}$ are either zero or one. We first analyze Step 6. If $z_{e+t} = 0$, then $(-1)^{z_1}\times 1\cdot0\ldots 01\times 2^{\z_{2:e+1}-q}$ is added to $(-1)^{y_1}\times 1\cdot\ldots 0\times 2^{\y_{2:e+1}-q}$, which gives $(-1)^{y_1}\times 1\cdot0\ldots \times 2^{\y_{2:e+1}-q-1}$ after rounding off. The remaining bits remain unchanged. A similar thing happens when both $z_{e+t}$ and $y_{e+t}$ are ones.

Therefore, after Step 6, the summation is $(-1)^{y_1} 1\cdot 0y_{e+2}\ldots y_{e+t-1}\times 2^{\y_{2:e+1}-q-1}$, after which Steps 7 and 8 alternately subtract $(-1)^{y_1}\times 1\cdot 0\ldots 0 z_{j}\times 2^{\y_{2:e+1}-q}$ and add $(-1)^{y_1}\times 2^{\y_{2:e+1}-q}$, which finally gives zero.

\textbf{Case 2: Min $j$ such that $y_j \neq z_j$ is less than $e+t$.}\\
In this case, we want to show that the final summation will be negative. Adding Step 6 to $(-1)^{y_1} 1\cdot y_{e+2}\ldots y_{e+t-1}\bar y_{e+t}\times 2^{\y_{2:e+1}-q}$ gives either $(-1)^{y_1} 1\cdot0 y_{e+2}\ldots y_{e+t-1}\times 2^{\y_{2:e+1}-q-1}$ or $(-1)^{y_1} (1\cdot 0y_{e+2}\ldots y_{e+t-1} + 2^{-(t-1)})\times 2^{\y_{2:e+1}-q-1}$, which is then subtracted by $-(-1)^{z_1}\times 2^{\z_{2:e+1}-q}$. The $2^{-(t-1)}$ is added when $y_m = 0$ and $z_m = 1$. Due to the existence of $j$, the summation will be negative as $y_{j+1}\ldots y_{e+t-1}+ 1 <1z_{j+1}\ldots z_{e+t-1}$ (as per the promise, $y_{e+t-1}y_{e+t}\neq 10$).

\textbf{Case 3: $\mathbf{y}$ and $\mathbf{z}$ differ only at the last bit.}\\
The only possibility in this case is $y_m = 0$ and $z_{m}=1$, which would make the left-associative summation $(-1)^{y_1}\times (1\cdot 0y_{e+2}\ldots y_{e+t-1} +2^{-(t-1)})\times 2^{\y_{2:e+1}-q-1}$. After adding Steps 7 and 8, the final value essentially becomes  $(-1)^{y_1}\times (1+2^{-(t-1)})\times 2^{\y_{2:e+1}-q-1}$, which is again a non-zero number.\\

Finally, a simple MLP having 2 hidden nodes in a hidden layer, similar to the construction given in the proof of Lemma \ref{lem:ub-fixed-pt}, can check if the output of self-attention is $0$, and returns 1, otherwise returns $0$. 
\end{proof}

\subsection{Tradeoff with Floating-Point Precision: Softmax Self-Attention}

We now show that using fixed-point precision $(t, e)$ where $t+e-1 = \mm + \A$ bits (one sign bit each), we can also solve a function, that can not be solved using any less number of bits. We assume the inputs 
\begin{itemize}
    \item The binary values satisfy $\y \leq \z$.
    \item $z_1, \dots z_e \neq 1\ldots 1$.
\end{itemize}

The input tokens are again similar to the previous case, where one part of the input token will be used to create the exponent in the floating-point precision format. For a given $(y_1, \ldots, y_{m}), (z_1, \ldots, z_m)$, where $m$ is odd:
\begin{itemize}
    \item $\Tm^3(\y, \z)_0 = ((y_1, \ldots, y_{e-1}), 0)$.
    \item $\Tm^3(\y, \z)_i = ((y_1, \ldots, y_{e-1}), y_{e+i-1})$ for $i\in [t-1]$.
    \item $\Tm^3(\y, \z)_{t-1+i} = ((y_{t+e-1}, \ldots, y_{t+2e-2-\A}), y_{t+2e-2-\A+i})$ for $i \in [t-1]$.
    \item $\Tm^3(\y, \z)_{2t-1} = ((y_{t+e-1}, \ldots, y_{t+2e-2-\A}), 0)$.
    \item $\Tm^3(\y, \z)_i = ((0), 0)$ for $i \in [2t, m-1]$.
    \item $\Tm^3(\y, \z)_{m} = ((z_1, \ldots, z_{e-1}), 0)$.
    \item $\Tm^3(\y, \z)_{m+i} = ((z_1, \ldots, z_{e-1}), z_{e+i-1})$ for $i\in [t-1]$.
    \item $\Tm^3(\y, \z)_{m+t-1+i} = ((z_{t+e-1}, \ldots, z_{t+2e-2-\A}), z_{t+2e-2-\A+i})$ for $i \in [ t-1]$.
    \item $\Tm^3(\y, \z)_{m+i} = ((0), 0)$ for $i \in [2t-1, m]$.
   \item $\Tm^3(\y, \z)_{2m+1} = ((z_{t+e-1}, \ldots, z_{t+2e-2-\A}),0)$
\end{itemize}

For ease of notation, we write $ \y_{1:e-1}$  as $e^1_y$, $ \y_{t+e+1:t+2e-\log \mm}$ as $e^2_y$, $\z_{1:e-1}$  as $e^1_z$ and $\z_{t+e+1:t+2e-\log \mm}$ as $e^2_z$.

Again, we choose $m := 2p-1$ and prove the following results.

\begin{lemma}[Communication Lower Bound]
        If Alice has $\y$ and Bob has $\z$ in the one-way communication where only Alice can send messages to Bob, Alice needs to send at least $m$ bits to Bob in order for Bob to compute the function $g(\y, \z)$.
\end{lemma}

\begin{proof}
    The proof is exactly the same as that of lemma \ref{lem:cc-F}.
\end{proof}

\begin{lemma}[Transformer Lower Bound]
\label{lem:lb-exp}
    Any linear Transformer having $d_v=1$ and $H=1$, that computes $T_m$ requires precision $p \geq m$ bits.
\end{lemma}

\begin{proof}
    This lower bound is again exactly same as the proof of Lemma \ref{lem:lb}, but Alice's set of tokens, $S$, is now $\{0, 1, \ldots, m-1\}$. Using this, we can show the Transformer requires precision at least $\mm$.
\end{proof}

\begin{lemma}[Upper Bound]
\label{lem:ub-exp}
    There exists a Transformer which can compute the function $T_m$, with floating-point precision, using precision bits  $p = (t, e)$ (one in the exponent is the sign bit for each mantissa and exponent), with $t+e = \mm + \log \mm$ and $e > \log \mm + 2$, $t > 2$.
\end{lemma}

\begin{proof}
The idea of the proof is similar to the previous construction, but with a slight modification where each token $\Tm^3(\y, \z)_i$ is a tuple $(\Tm^3(\y, \z)^1_i, \Tm^3(\y, \z)^2_i)$, where $\Tm^3(\y, \z)^1_i \in \{0, 1\}^{e-1}$ and $\Tm^3(\y, \z)_i^2 \in \{0, 1\}$.

\paragraph{Overview.} We again prove this upper bound on the function $T_m$ by checking equality of the first half of $\y$ and $\z$ in the numerator, and the rest in the denominator. However, we obtain a slightly weaker tradeoff in this lemma, for which we check equality of $\y_{1:\mm+\A}$ and $\z_{1:\mm+\A}$ (slightly more than half) in the numerator, and that of $y_{\mm+\A:m}$ and $\z_{\mm+\A:m}$ in the denominator.

The idea is again similar to the previous constructions, where the numerator computes $$1\cdot y_{e}\ldots y_{e+t-1}\times 2^{e^1_y} - 1\cdot z_{e}\ldots z_{e+t-1}\times 2^{e^1_z},$$ along with a normalizing term to indicate when the first half if similar. The denominator will have terms from the first half which will be ignored, and the second half will be a fixed number $(1-2^{(t-1)})\times 2^{e^2_y}$ if and only if the second half is also equal.

\paragraph{Embedding.} We define the embedding as follows:

\begin{table}[H]
\resizebox{\columnwidth}{!}{%
\begin{tabular}{c|c|c|c}
    Step& $i$ & $\Tm^3(\y, \z)^2_i = 0$ & $\Tm^3(\y, \z)^2_i=1$ \\ \hline
    1 & $i = 0$ & $\begin{bmatrix}
        1 && {-(2^{e-1}-1)} && 1\cdot0\times 2^{e^1_y}
    \end{bmatrix}$ & NA\\
    2 & $i\in [t-1]$ & $\begin{bmatrix}
        1 && -(2^{e-1}-1) && 0
    \end{bmatrix}$ & $\begin{bmatrix}
        1 && -(2^{e-1}-1) && 2^{-i}\times 2^{e^1_y}
    \end{bmatrix}$\\
    3 & $i \in [t, 2t-2]$ & $\begin{bmatrix}
        1 && -(i-t+1)-e^2_y && 0
    \end{bmatrix}$ & $\begin{bmatrix}
        1 && -N && 0
    \end{bmatrix}$\\
    4 & $i=2t-1$ & $\begin{bmatrix}
        1 & -e^2_y & 0
    \end{bmatrix}$ & NA\\
    5 & $i \in [2t, m-1]$ & $\begin{bmatrix}
        1 && -N && 0
    \end{bmatrix}$ & NA\\
    6 & $i = m$ & $\begin{bmatrix}
        1 && -(2^{e-1}-1) && -1\cdot0\times 2^{e^2_y}
    \end{bmatrix}$ & NA\\
    7 & $i\in [m+1, m+t-1]$ & $\begin{bmatrix}
        1 && -(2^{e-1}-1) && 0
    \end{bmatrix}$ & $\begin{bmatrix}
        1 && -(2^{e-1}-1) && -2^{-(i-m)}\times -2^{e^1_z}
    \end{bmatrix}$\\
    8 & $i \in [m+t, m+2t-2]$  & $\begin{bmatrix}
        1 && -N && 0
    \end{bmatrix}$ & $\begin{bmatrix}
        1 && -(i-m-t+1)-e^2_z && 0
    \end{bmatrix}$\\
    9 & $i \in[m+2t-1, 2m]$ & $\begin{bmatrix}
        1 && -N && 0
    \end{bmatrix}$ & NA\\
    10 & $i = 2m+1$ & $\begin{bmatrix}
        1 && -(2^{e-1}-1) && 1\cdot0\times 2^{-e^2_z}
    \end{bmatrix}$ & NA\\
    11 & $i > 2m+1$ & $\begin{bmatrix}
        1 && -N && 0
    \end{bmatrix}$ & NA
\end{tabular}%
}
\label{tab:pe-3}
\caption{Embedding}
\end{table}

\paragraph{Construction of the Weights.} We define the query weights $W^Q \in \R^{d\times d_q}$, for query dimension  $d_q = 1$, as,
    \begin{equation*}
        W^Q = \ln 2\begin{bmatrix}
             1  \\
             0  \\
             0   
        \end{bmatrix},
    \end{equation*}
    the key weights $W^K\in \R^{d\times d_q}$, as,
    \begin{equation*}
        W^K = \begin{bmatrix}
              0 \\
              1 \\
              0 
        \end{bmatrix},
    \end{equation*}
    and the value weights $W^V \in \R^{d\times d_v}$, for value dimension $d_v = 1$, as,
    \begin{equation*}
        W^V = \begin{bmatrix}
            
            0 \\
            0 \\
            2^{2^{e-1}-1}
        \end{bmatrix}.
    \end{equation*}

\paragraph{Numerator.} For computing the numerator, as before we compute each of $e^{\langle Q_{n+1} K_i\rangle}X_iW^V$ by left-associative multiplication, after which we exponentiate.

From the embeddings, Step 1 gives $1\cdot 0\times 2^{e^1_y}$, after which Step 2 adds $y_{e+i-1}2^{-i}\times 2^{e^1_y} = y_{e+i-1}\times 1\cdot 0\times 2^{e^1_y-i}$, for $i\in [t-1]$, which is also in the permissible range. Therefore, Step 1 and 2 combined give $1\cdot y_{e}\ldots y_{e+t-2}\times 2^{e^1_y} $. 

Steps 3-5 are zeros. Step 6 subtracts $1\cdot 0\times 2^{e^1_z}$ from this. If $\y_{1:e-1} < \z_{1:e-1}$, then after Step 6, the value becomes negative (since $1\cdot 0\times 2^{e^1_z} > 1\cdot y_e\ldots y_{e+t-2}\times 2^{2^1_y}$), and continues to stay negative as Step 7 essentially generates the number $-0\cdot z_{e}\ldots z_{e+t-2}\times 2^{e^1_z}$, after which everything else is zero. There is no loss due to rounding-off since $t-1 < 2^{e-1}$ (since $e > \log \mm$ and $t+e = \mm + \log \mm$).

Now, assuming $e^1_y = e^1_z$, after Step 6, the summation becomes $0\cdot y_{e}\ldots y_{e+t-2}\times 2^{e^1_y}$, which is still within the range of the allowed precision bits because $t-1 < 2^{e-1}$. After this, in Step 7, $-z_{e+j}\times 2^{-j}\times 2^{e^1_z}$ is added from this summation, for $j\in [t-1]$, which gives the number  $(1\cdot y_{e}\ldots y_{e+t-2}-1\cdot z_e\ldots z_{e+t-2})\times 2^{e^1_z}$.  If $\y_{e+t-2} = \z_{e+t-2}$, then this value is zero, otherwise it is some negative number (since $\y < \z$), with however a positive exponent (due to $t-1 < 2^{e-2}$). 

Steps 8, 9, 11 add 0 to this, and Step 10 makes the summation equal to $1\cdot 0\times 2^{-e^2_z}$ if $\y_{1:e+t-2}=\z_{e+t-2}$, otherwise some negative number (the exponent of $(1\cdot y_{e}\ldots y_{e+t-2}-1\cdot z_e\ldots z_{e+t-2})\times 2^{e^1_z}$ is positive).

\paragraph{Denominator.} For denominator, we compute each of $e^{X_{n+1}W^Q (W^K)^T X_i}$. Steps 1, 2 yield the summation as $t\times 1\cdot 0\times 2^{-(2^{e-1}-1)} = 1\cdot 0\times 2^{\log t - (2^{e-1}-1)}$, where $1\cdot 0\times 2^{-(2^{e-1}-1)}$ is the smallest possible positive number. To this, Step 3 adds $\bar y_{e+2t-2 - \A +j}2^{-(j)}\times 2^{-e^2_y}$, and along with Step 4, gives the number $0\cdot \bar y_{e+2t-\A-1}\ldots \bar y_m \times 2^{-e^2_y} + 1\cdot 0\times 2^{\log t - (2^{e-1}-1)}$. Now, $e^2_y$ consists of $e-1-\A$ bits, which implies 
\begin{equation*}
    \begin{split}
        &2^{e-1}-1-\log t - e^2_y > 2^{e-1}-1-\log t-2^{e-1-\A}-1 \\
        & = 2^{e-2} + ((1-\frac{1}{\mm})2^{e-2} - \log t -1) > t. 
    \end{split}
\end{equation*}
This implies that $1\cdot 0\times 2^{\log t - (2^{e-1}-1)}$ will be ignored while rounding off, giving the final value as $0\cdot \bar y_{e+2t-\A-1}\ldots \bar y_m \times 2^{-e^2_y}$. 

Step 5 simply adds zero and have been used as dummy tokens for notational ease.

Step 6-7 again forms $1\cdot 0\times 2^{\log t - (2^{e-1}-1)}$, which will be rounded off, and Step 8 adds $z_{e+2t-2 - \A +j}2^{-(j)}\times 2^{-e^2_z}$, which gives us $0\cdot \bar y_{e+2t-\A-1}\ldots \bar y_m \times 2^{-e^2_y} + 0\cdot  z_{e+2t-\A-1}\ldots z_m \times 2^{-e^2_z}$ as the final answer.

\paragraph{Wrap up.} We study the following cases, by first computing the value of the output of the self-attention for the equality case. When $y_1\ldots, y_m = z_1\ldots z_m$: 
\begin{itemize}
    \item The numerator will become $1\cdot y_{e}\ldots y_{t+e-2}\times 2^{e^1_y} - 1\cdot z_{e} \ldots z_{t+e-2}\times 2^{e^1_z} + 1\cdot0\times 2^{-e^2_z}$, which is equal to just $1\cdot0\times 2^{-e^2_y}$.
    \item For the denominator, the value will essentially become $0\cdot \bar y_{t+2e-1-\A}\ldots \bar y_{m}\times 2^{-e^2_y}  + 0\cdot z_{t+2e-1-\A}\ldots z_{m}\times 2^{-e^2_z} $, which is equal to $1\cdot\underbrace{1\ldots 1}_{t-2}\times 2^{-(e^2_z+1)}$.
    \item This gives the softmax as $\frac{1}{1-2^{-(t-1)}}$.
\end{itemize}

We now show that if equality does not hold, the self-attention's output will always be different, and can be checked using the given bits of precision.

\textbf{Case 1, $y_1\ldots y_\mm \neq z_1\ldots z_\mm$:}

\textbf{Case 1.1, $e^1_y < e^1_z$:}\\
In this case, the value of the numerator will be negative  as discussed before.

\textbf{Case 1.2, $e^1_y = e^1_z$ but $y_{e}\ldots y_{e+t-1} < z_{e}\ldots z_{e+t-1}$:}\\
Again in this case, we the numerator will be negative.\\

\textbf{Case 2, $y_1\ldots y_\mm = z_1\ldots z_\mm$ but $y_1\ldots y_m \neq z_1\ldots z_m$:}\\
In this case, the numerator will always be $1\cdot0\times 2^{-e^2_z}$. Let us now analyze the denominator.

\textbf{Case 2.1, $e^2_y < e^2_z$:}\\
The denominator will be $0\cdot \bar y_{t+2e-1-\A}\ldots \bar y_{m}\times 2^{-e^2_y}  + 0\cdot z_{t+2e-1-\A}\ldots z_{m}\times 2^{-e^2_z} < 1\cdot 1\times 2^{-e^2_y}$, and  the final output as $\geq \frac{2^{e^2_y - e^2_z}}{1 + 2^{-1}} $, which is strictly greater than the equality case where the value is $\frac{1}{1-2^{-(t-1)}}$ for $t > 2$.

\textbf{Case 2.2, $e^2_y = e^2_z$ but $y_{t+2e-1-\A}\ldots y_{m}$:}\\
In this case, there will be a maximum $i_0$ such that for all $i\in [t+2e-1-\A, i_0-1]$, $y_i = z_i$. Due to the condition that $\mathbf{y} < \mathbf{z}$, we must have $y_{i_0} = 0$ and $z_{i_0} = 1$. Therefore, from the value of the denominator, the exponent $2^{-e^2_y}$, and for the mantissa, the summation will be $0\cdot \underbrace{1\ldots 1}_{i_0-1} \ldots \times 2^{-e^2_y}$, while for the $i_0$-th place after decimal, Step 3 and Step 8 will both give $2^{-i_0}\times 2^{-e^2_y}$. Therefore, the value of the denominator will be strictly greater than $1\cdot1\ldots 1\times 2^{-e^2_y}$, and , the output of the self-attention will also be $< \frac{1}{1-2^{-(t-1)}}$. This will differ from the expected output at the $i_0$-th position after decimal. Since $i_0 \leq t$, this can be checked.\\

Finally, we can again construct an MLP similar to the one used in the proof of Lemma \ref{lem:ub-fixed-pt}, containing only one hidden layer which has only 2 nodes, that returns 0 if the output of self-attention is not equal to $\frac{1}{2-2^{-(t-1)}}$, and $1$ otherwise.
\end{proof}

    \begin{proof}[Proof of Theorem \ref{thm:tradeoff-fp}]
            We choose $m := 2p-1$ and define $\Gamma^{(t, e)}_n$ as $T_m$. From Lemma \ref{lem:lb-exp}, $\Gamma^{(p)}_n$ can not be computed by a Transformer using $\leq \mm-1 = p-1$ bits, but can be computed by a Transformer using $\mm + \A +2 = p + \log p +2$ bits (Lemma \ref{lem:ub-exp}).
    \end{proof}

\section{Details of Experiments}
\label{apx:exp_eq}

In this section, we complete the details of the experiments. A one-layer Transformer was trained to learn the equality task, and then quantized to both fixed-point and floating-point precisions using PTQ. After this, each model was further fine-tuned using QAT to report the final accuracies. The experiments were performed over 10 seeds randomly generated.

The experiments were repeated for $4$ values of $m$ with $m=100, 50, 30, 15$.

\paragraph{Input Generation.} For each $m$, the tokens are of length $n = 2m+1$, where the first $m$ tokens contain the first string $\y \in \{0, 1\}^m$, the next $m$ tokens contain the second string $\z \in \{0, 1\}^m$, and the final token contains zero. The output is expected in the final token.

The strings are chosen such that:
\begin{enumerate}
    \item With probability $1/2$, the strings will be equal. In this case, $\y$ is chosen uniformly at random and $\z$ is set to $\y$.
    \item With the remaining $1/2$ probability, the strings will be unequal. In this case, $\y$ is chosen uniformly at random, and a set $P\subset [m]$ of size $|P| = \lfloor 0.75 m\rfloor$ is also chosen uniformly at random. Then, $\z$ is set of the string whose bits are equal to that of $\y$ at indices $[m]\backslash P$, and are flipped at indices in $P$.
\end{enumerate}

\paragraph{Architecture Details.} A one-layer Transformer was trained on this dataset, defined as exactly similar to Definition \ref{def:soft-SA}. The Transformer had heads $H = 2$, embedding dimension $d = 8$ (each head had hidden dimension 4), and the MLP had a hidden layer of size $32$. 

The $i$-th token containing the value $X_i$ was encoded as the integer $i+n X_i$, for $i\in \{0, \ldots, n-1\}$. After this, it was fed to an embedding function, followed by addition with the sinusoidal positional encoding of \cite{vaswani2017attention}. Layer norms were present after the soft-max and after the MLP layer.

\paragraph{Training Dynamics.} The model was trained with AdamW optimizer of PyTorch optim module (\textsf{torch.optim}\footnote{https://pytorch.org/}) with weight decay being 0 and other configurations set to the default values, over $\fb$ precision. The learning rate was chosen to be $0.001$ and the model was trained for 30,000 epochs for $m=100$, 20,000 epochs for $m=50$, 6,000 epochs for $m=30, 15$, each with a batch size of 512. The test was performed every 50 epochs with a fresh batch having batch size 5120.

The learning curves (loss $\pm$ standard deviation and test accuracy $\pm$ standard deviation), for each $m$, averaged over 10 seeds, can be found as follows.

The models were trained on NVIDIA A100 GPUs having 80 GB GPU RAM.

\begin{figure}[H]
    \centering
    \includegraphics[width=\linewidth]{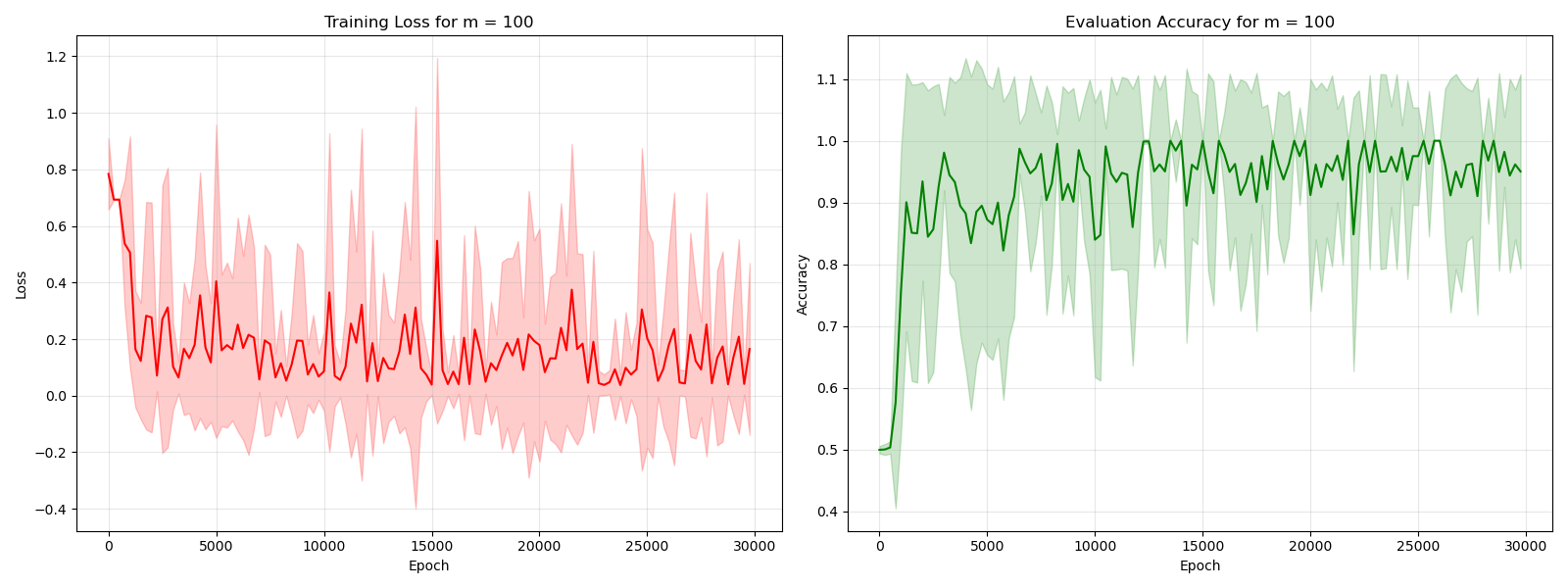}
    \caption{Learning curves with mean and one standard deviation for $m=100$.}
    \label{fig1}
\end{figure}
\begin{figure}[H]
    \centering
    \includegraphics[width=\linewidth]{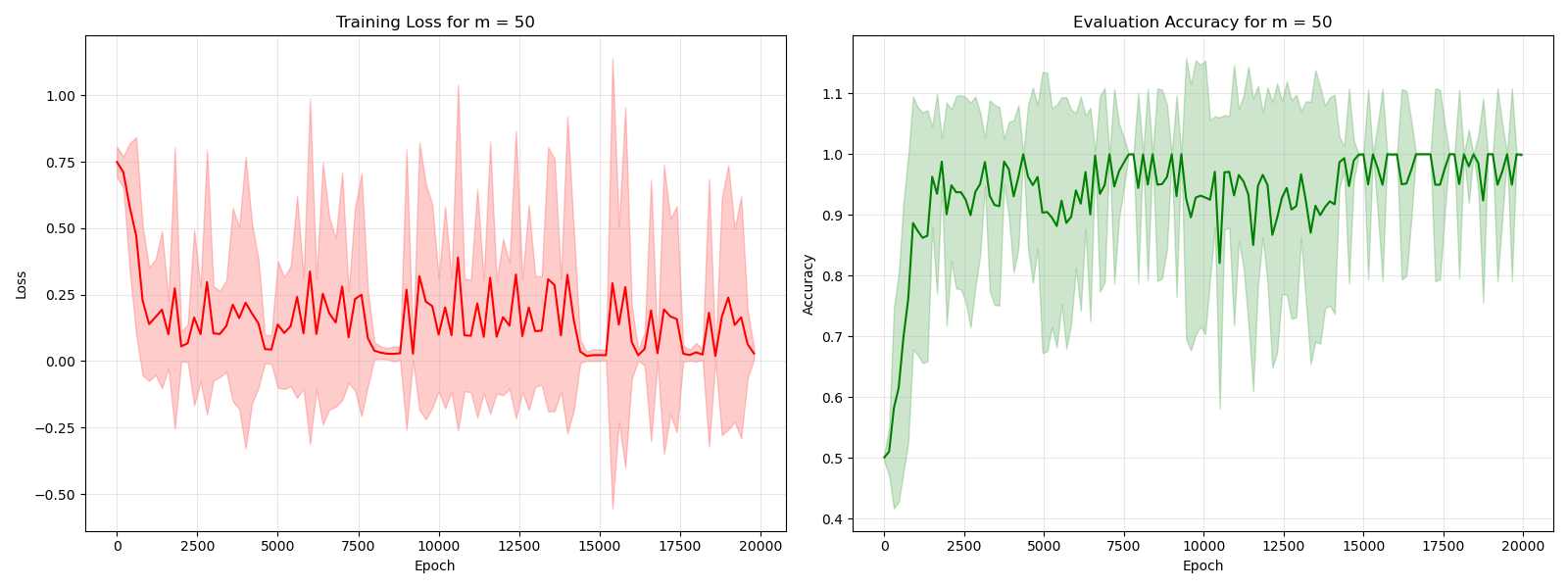}
    \caption{Learning curves with mean and one standard deviation for $m=50$.}
    \label{fig2}
\end{figure}
\begin{figure}[H]
    \centering
    \includegraphics[width=\linewidth]{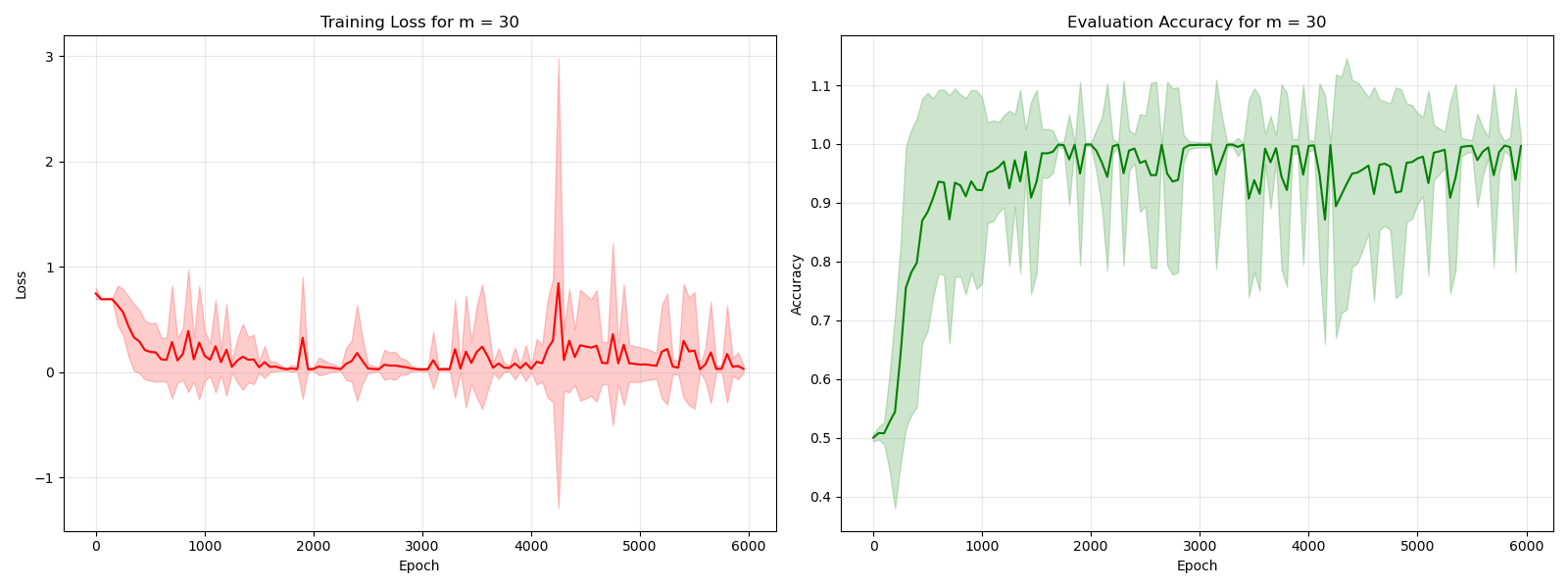}
    \caption{Learning curves with mean and one standard deviation for $m=30$.}
    \label{fig3}
\end{figure}
\begin{figure}[H]
    \centering
    \includegraphics[width=\linewidth]{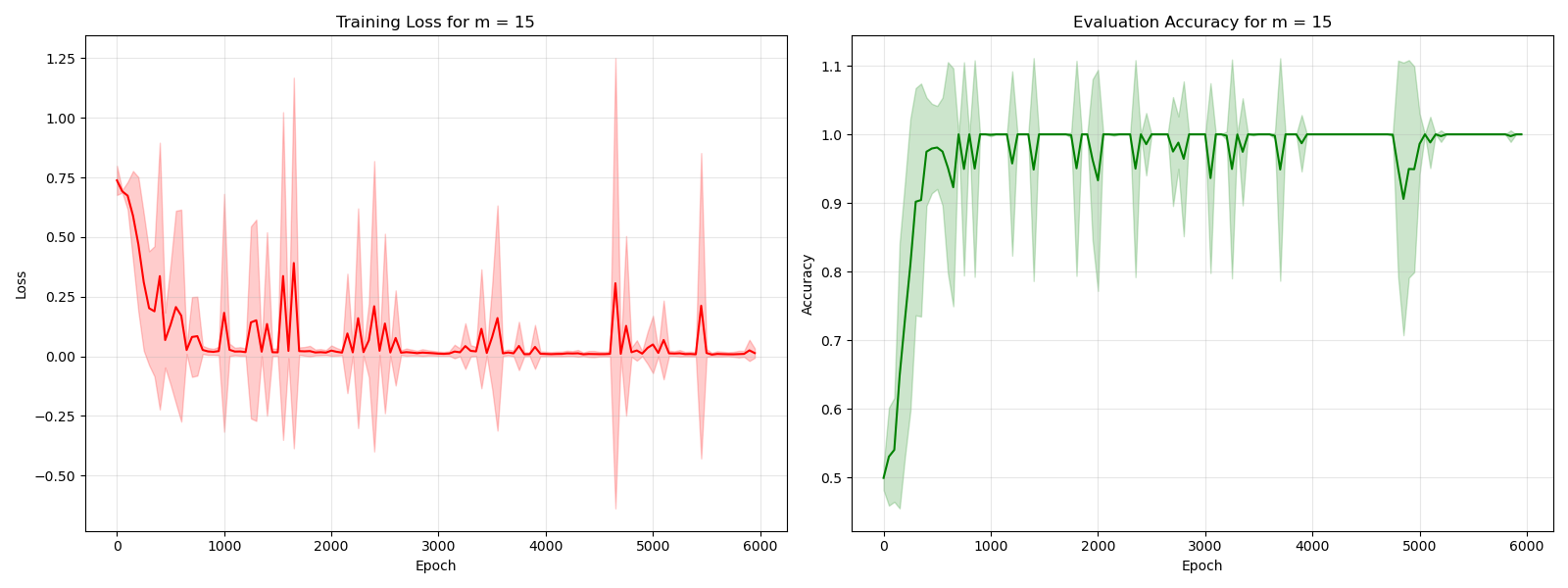}
    \caption{Learning curves with mean and one standard deviation for $m=15$.}
    \label{fig4}
\end{figure}

\paragraph{Quantization.} Once the model was trained, first PTQ was performed for fixed-point precisions (INT12, INT8, INT6, INT4) and floating-point precisions (FP16, FP\_E4M3, FP\_E5M2). 

Floating-point precision was performed by choosing a valid scale (within the range $[-2^{p-1}:2^{p-1}]$), and then rounding the number to $p-1$ bits and one sign bit. Fixed-point precision was performed by bringing the number in the specific precision format. Both precision formats had the maximum possible number allocated as infinity.

For quantizing the model, all the weights and activations were quantized, and the accuracies computed over a batch of size 5120 for 10 random seeds.

After PTQ, the model underwent QAT for 500 epochs, each epoch having the same training dynamics as before.  At each epoch, the loss function and the gradients were computed in $\fb$, and quantized after that at each step. The accuracies were also computed during QAT every 5 steps with a batch size of $5120$.

The QAT curves are shown as follows, containing the training loss and one standard deviation, and the test loss and one standard deviation, averaged over 10 random seeds. This shows that the model does not get much better even after the fine-tuning, hinting that the best fitting weights in the precision format has already been obtained.

\begin{figure}[H]
    \centering
    \includegraphics[width=\linewidth]{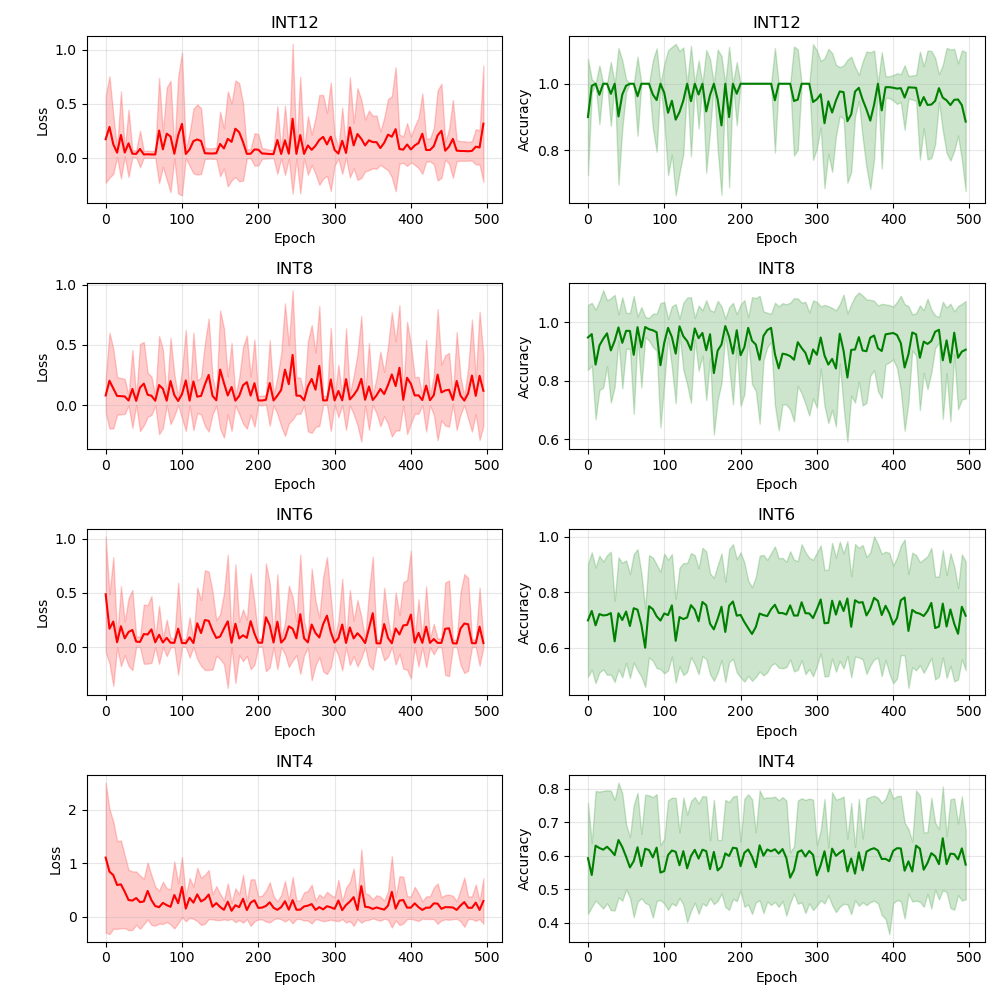}
    \caption{QAT curves for $m=100$ using fixed-point precision}
    \label{QAT1}
\end{figure}
\begin{figure}[H]
    \centering
    \includegraphics[width=\linewidth]{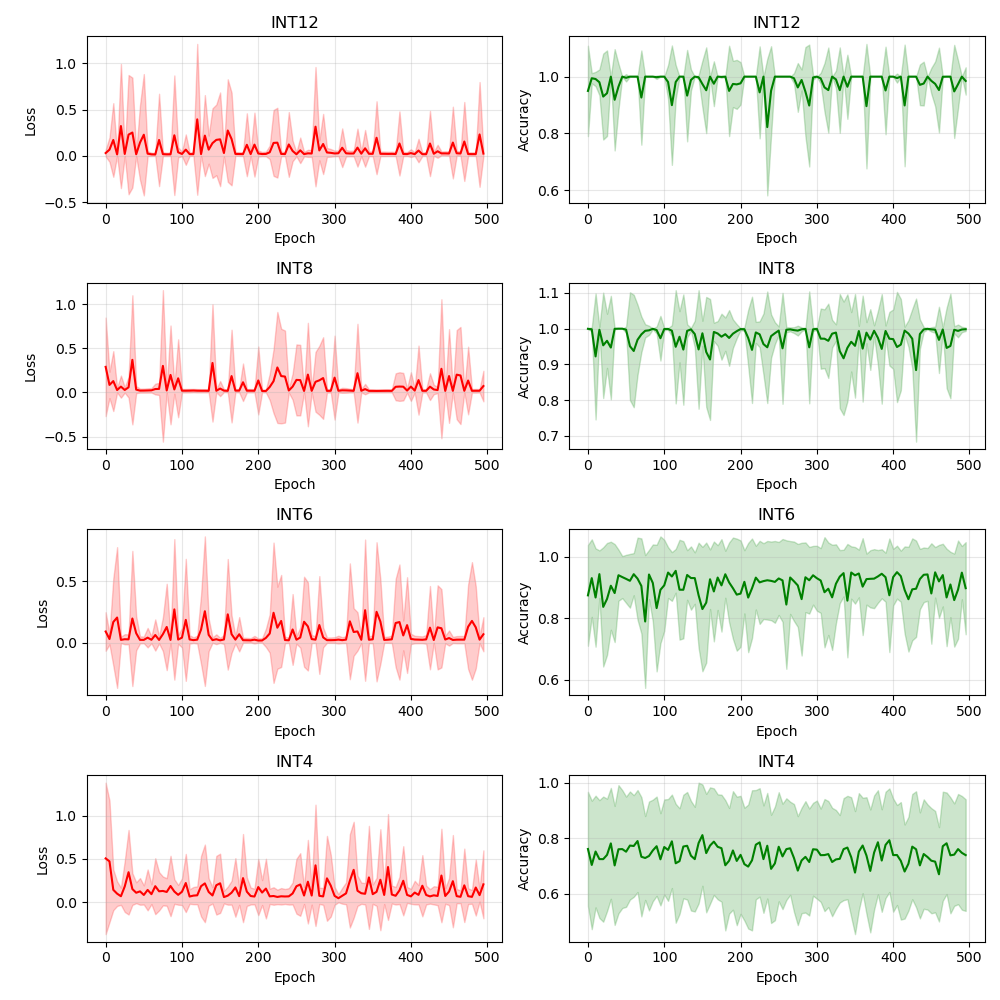}
    \caption{QAT curves for $m=50$ using fixed-point precision}
    \label{QAT2}
\end{figure}
\begin{figure}[H]
    \centering
    \includegraphics[width=\linewidth]{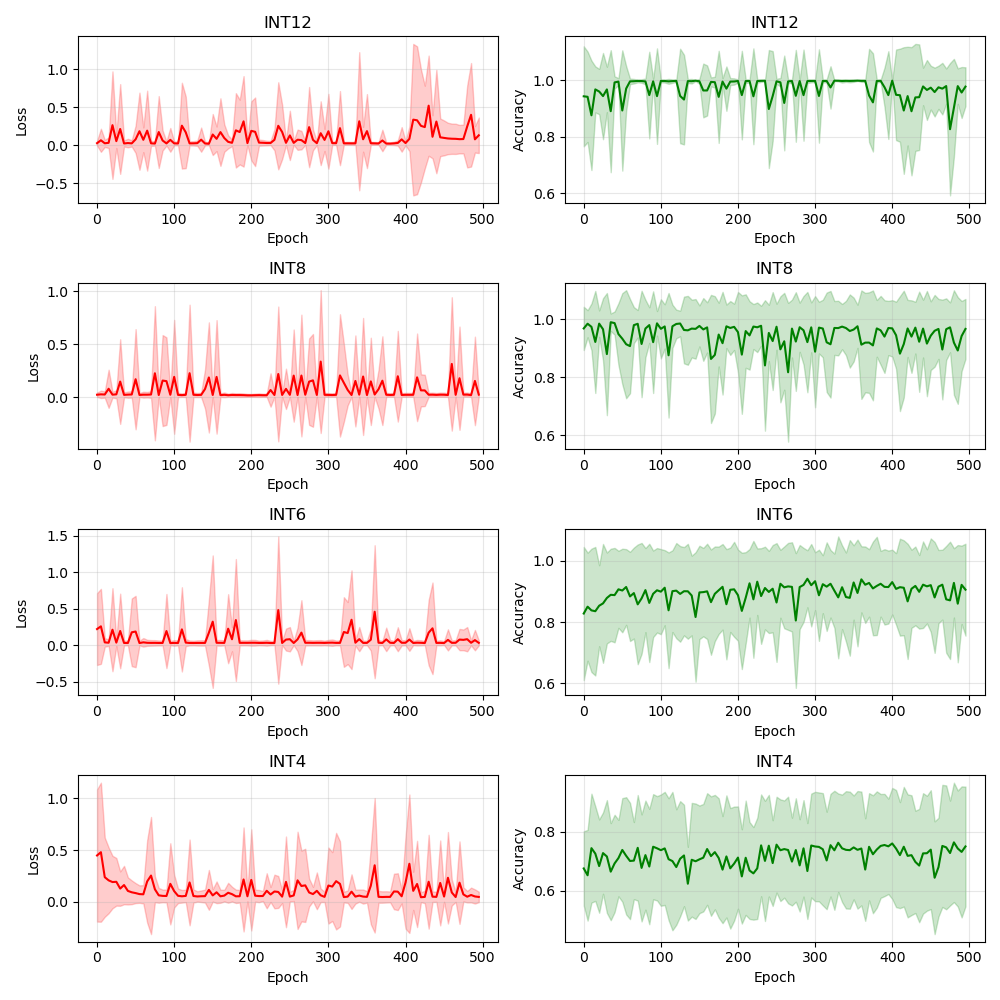}
    \caption{QAT curves for $m=30$ using fixed-point precision}
    \label{QAT3}
\end{figure}
\begin{figure}[H]
    \centering
    \includegraphics[width=\linewidth]{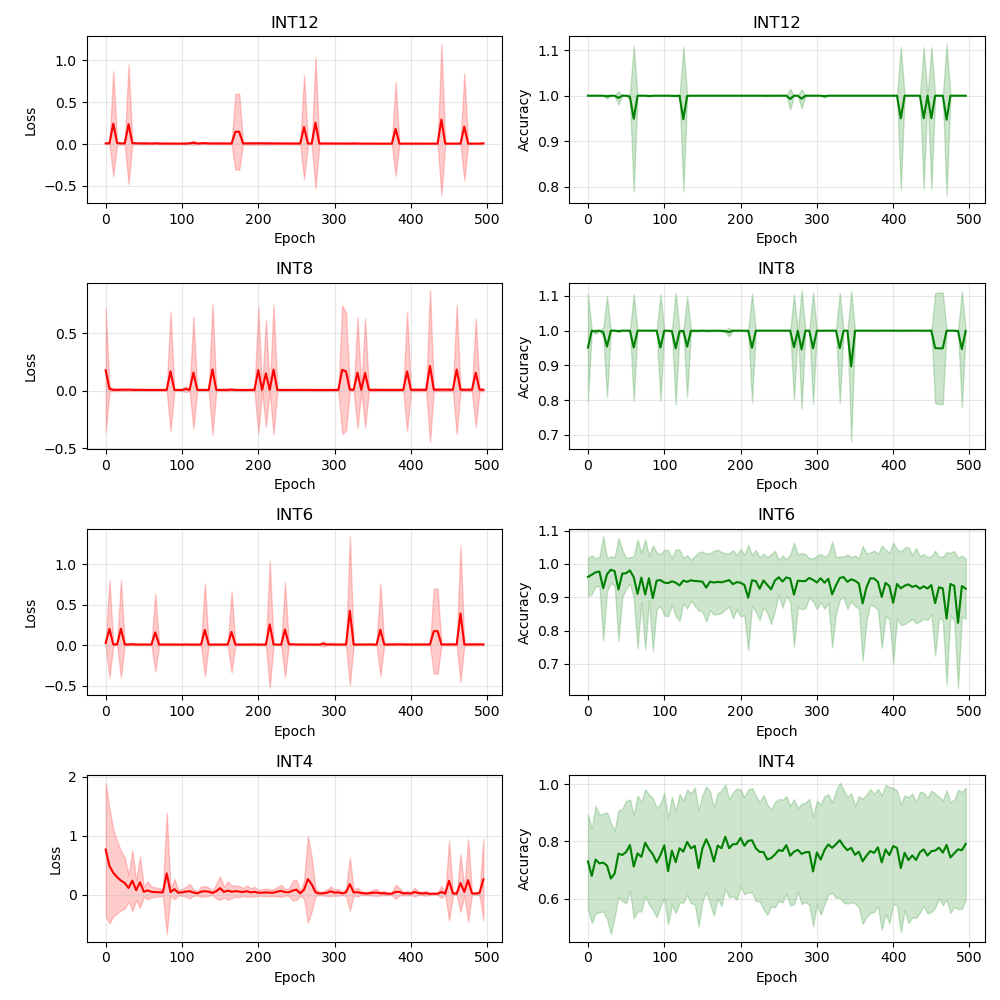}
    \caption{QAT curves for $m=15$ using fixed-point precision}
    \label{QAT4}
\end{figure}

\begin{figure}[H]
    \centering
    \includegraphics[width=\linewidth]{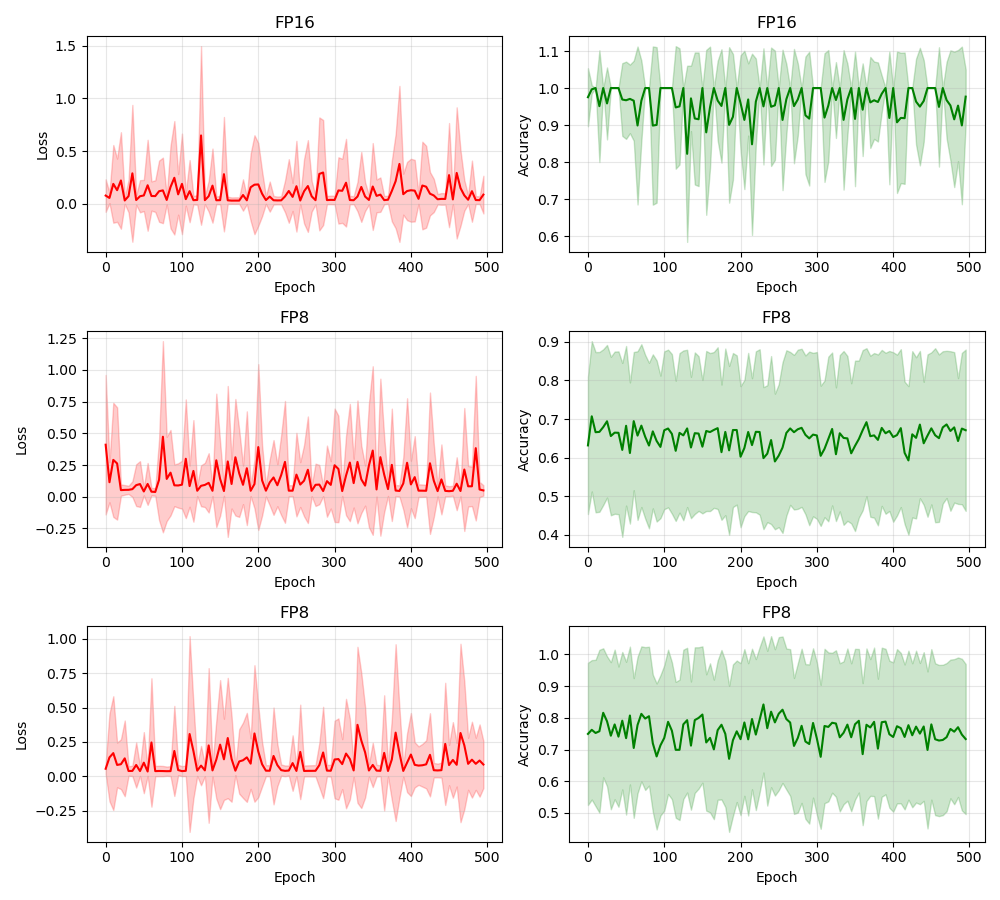}
    \caption{QAT curves for $m=100$ using floating-point precision}
    \label{QAT5}
\end{figure}
\begin{figure}[H]
    \centering
    \includegraphics[width=\linewidth]{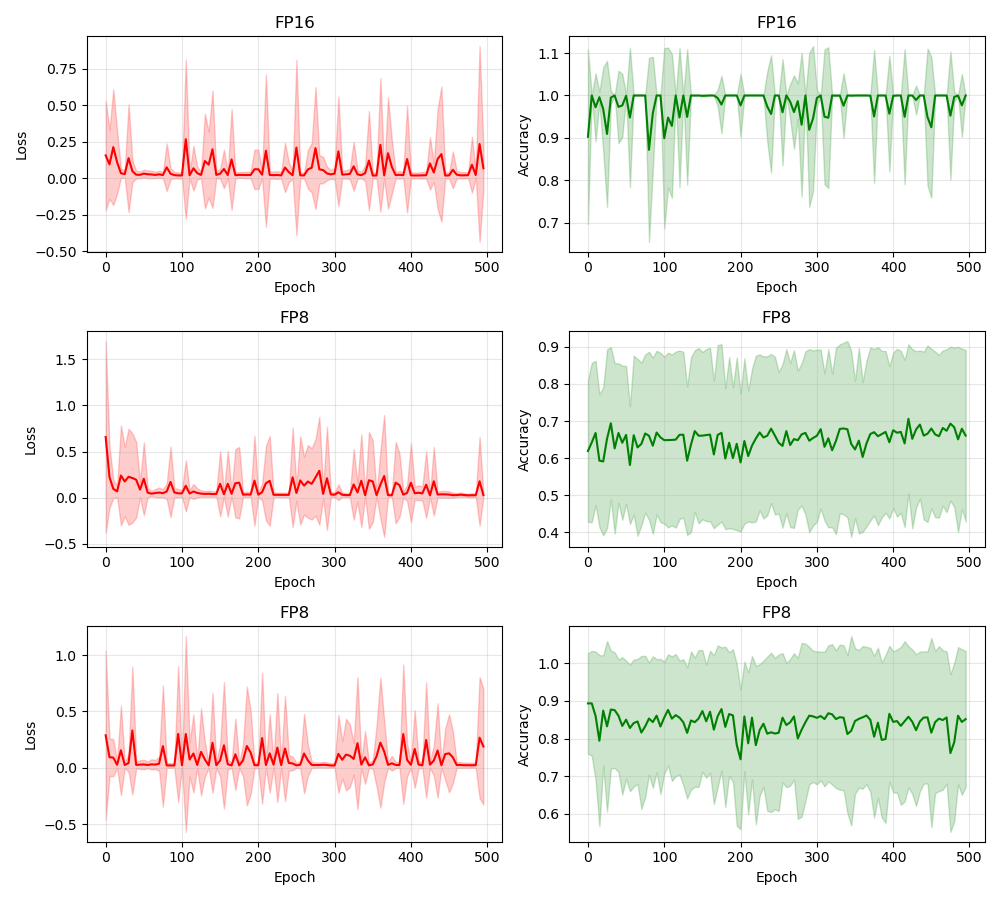}
    \caption{QAT curves for $m=50$ using floating-point precision}
    \label{QAT6}
\end{figure}
\begin{figure}[H]
    \centering
    \includegraphics[width=\linewidth]{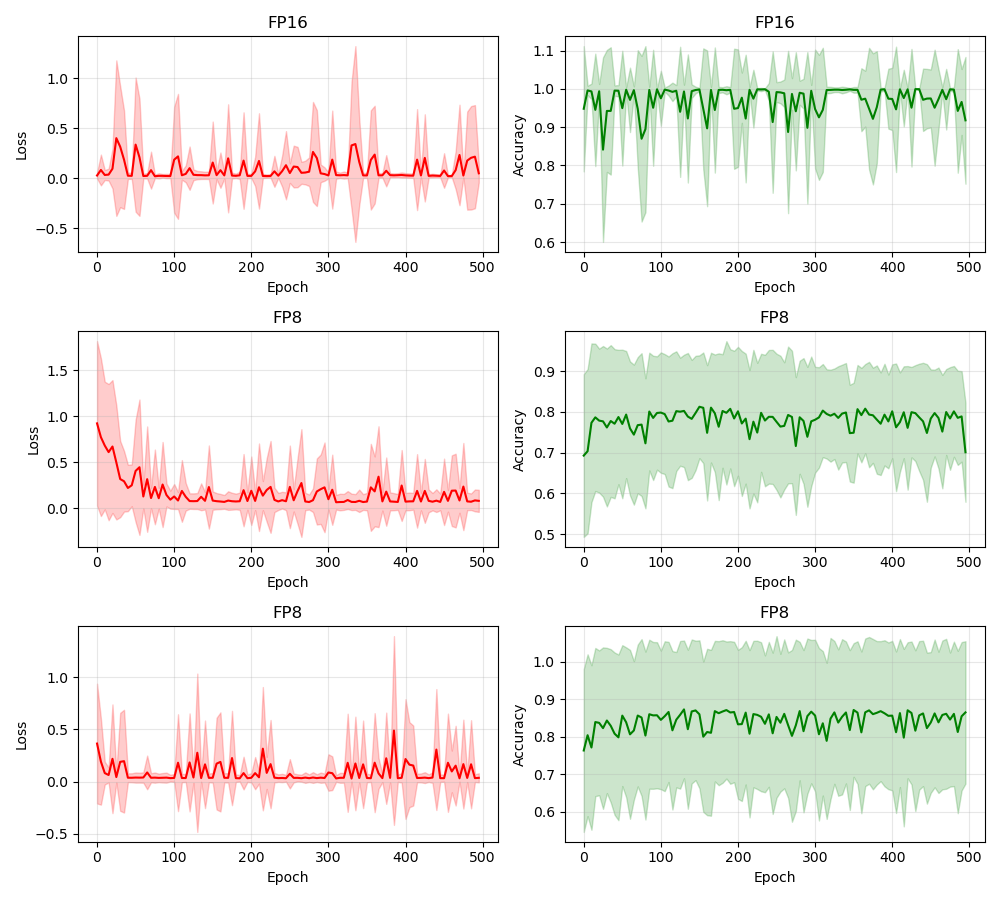}
    \caption{QAT curves for $m=30$ using floating-point precision}
    \label{QAT7}
\end{figure}
\begin{figure}[H]
    \centering
    \includegraphics[width=\linewidth]{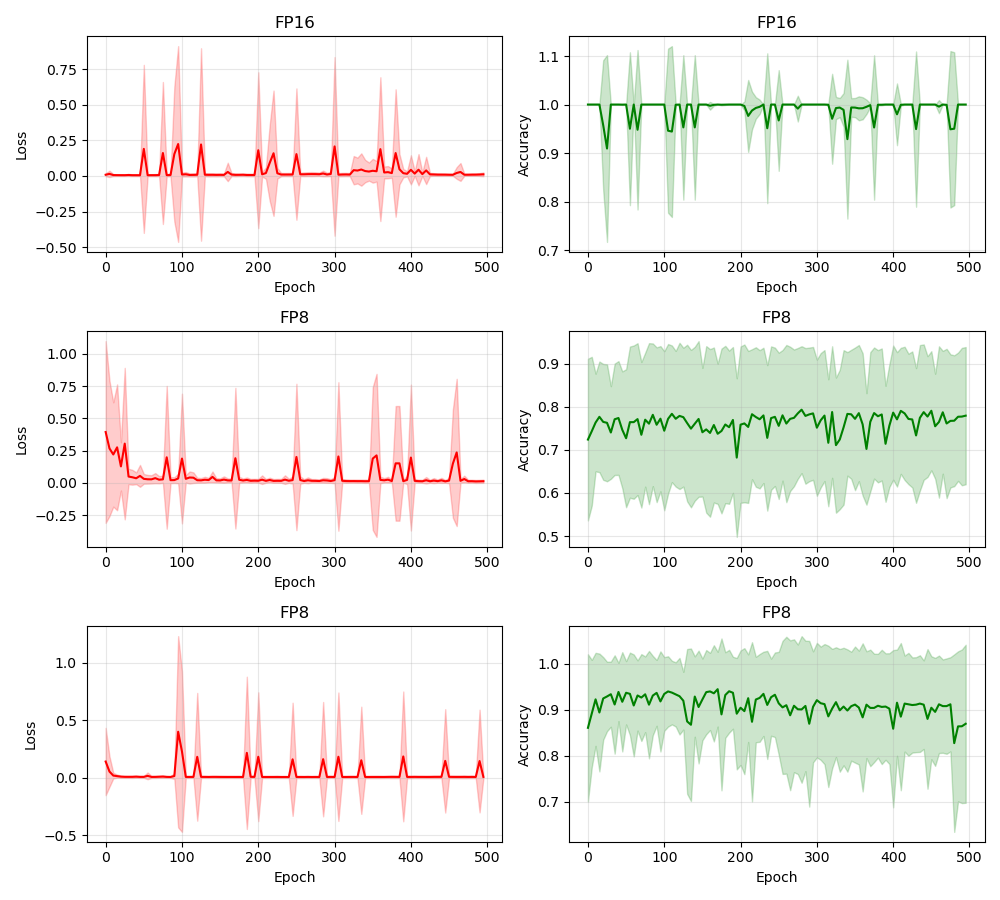}
    \caption{QAT curves for $m=15$ using floating-point precision}
    \label{QAT8}
\end{figure}

\end{document}